\pdfoutput=1
\documentclass{article}

\PassOptionsToPackage{numbers, compress}{natbib}

\usepackage{multirow}
\newcommand*\rot{\rotatebox{90}}
\usepackage{titlesec}
\newcommand{\sys}{{\sc OASIS}\xspace}

\usepackage[final,main]{neurips_2026}

\usepackage{caption}
\usepackage[utf8]{inputenc} %
\usepackage[T1]{fontenc}    %

\usepackage{setspace}

\usepackage{xspace}

\usepackage[colorlinks,citecolor=DarkGreen,linkcolor=FireBrick,urlcolor=FireBrick,linktocpage=true,unicode]{hyperref}
\usepackage{url}            %
\usepackage{booktabs}       %
\usepackage{amsfonts}       %
\usepackage{nicefrac}       %
\usepackage{
microtype}      %
\usepackage{amsmath}
\usepackage{amssymb}
\usepackage{mathtools}
\usepackage{slashed}
\usepackage{braket}
\usepackage{enumitem}
\usepackage{bbm}
\usepackage{setspace}
\usepackage[svgnames]{xcolor}
\usepackage{multirow}
\usepackage{subcaption}
\usepackage{graphicx}
\usepackage{graphicx}
\usepackage{tikz}
\usepackage{tikz-cd}

\usepackage{amsthm}
\usepackage{bm}
\usepackage{physics}
\usepackage{xcolor}
\usepackage{natbib}
\usepackage{mdframed}
\usepackage{nicefrac}
\usepackage{booktabs}
\usepackage{lipsum}
\usepackage{titlesec}
\usepackage{wrapfig,lipsum,booktabs}
\usepackage{blindtext}
\usepackage[most]{tcolorbox}

\newcommand{\Softmax}{\mathop{\rm{Softmax}}}
\newcommand{\Sparsemax}{\mathop{\rm{Sparsemax}}}

\definecolor{darkgreen}{rgb}{0,0.40,0}
\definecolor{firebrick}{rgb}{0.698,0.133,0.133}

\let\cite\citep

\definecolor{Blue}{rgb}{0, 0, 0.8}
\definecolor{blue}{rgb}{0,0,1}

\newtheoremstyle{theoremstyle}
  {.5\baselineskip} %
  {.5\baselineskip} %
  {}                  %
  {}                  %
  {\bfseries}        %
  {.}                 %
  {1em}               %
  {}                  %

\theoremstyle{theoremstyle}

\newtheorem{theorem}{Theorem}[section]
\newtheorem{lemma}{Lemma}[section]

\newtheorem{proposition}{Proposition}[section]
\newtheorem{definition}{Definition}[section]
\newtheorem{assumption}{Assumption}[section]

\tcolorboxenvironment{theorem}{
  breakable,
  colback=black!10,
  colframe=white,%
  width=\dimexpr\linewidth+10pt\relax,%
  enlarge left by=-5pt,%
  enlarge right by=-5pt,%
  boxsep=5pt,%
  boxrule=0pt,
  left=0pt,right=0pt,top=0pt,bottom=0pt,
  sharp corners,
  before skip=0.5\baselineskip, %
  after skip=0.5\baselineskip,  %
  fonttitle=\bfseries, %
  coltitle=black %
}
\tcolorboxenvironment{remark}{
  blanker,
  breakable,
  before skip=.8\baselineskip,  %
  after  skip=.8\baselineskip   %
}

\tcolorboxenvironment{proposition}{
  breakable,
  colback=black!10,
  colframe=white,%
  width=\dimexpr\linewidth+10pt\relax,%
  enlarge left by=-5pt,%
  enlarge right by=-5pt,%
  boxsep=5pt,%
  boxrule=0pt,
  left=0pt,right=0pt,top=0pt,bottom=0pt,
  sharp corners,
  before skip=0.5\baselineskip, %
  after skip=0.5\baselineskip,  %
  fonttitle=\bfseries, %
  coltitle=black %
}

\tcolorboxenvironment{lemma}{
  breakable,
  colback=black!10,
  colframe=white,%
  width=\dimexpr\linewidth+10pt\relax,%
  enlarge left by=-5pt,%
  enlarge right by=-5pt,%
  boxsep=5pt,%
  boxrule=0pt,
  left=0pt,right=0pt,top=0pt,bottom=0pt,
  sharp corners,
  before skip=0.5\baselineskip, %
  after skip=0.5\baselineskip,  %
  fonttitle=\bfseries, %
  coltitle=black %
}

\tcolorboxenvironment{corollary}{
  breakable,
  colback=black!10,
  colframe=white,%
  width=\dimexpr\linewidth+10pt\relax,%
  enlarge left by=-5pt,%
  enlarge right by=-5pt,%
  boxsep=5pt,%
  boxrule=0pt,
  left=0pt,right=0pt,top=0pt,bottom=0pt,
  sharp corners,
  before skip=0.5\baselineskip, %
  after skip=0.5\baselineskip,  %
  fonttitle=\bfseries, %
  coltitle=black %
}

\tcolorboxenvironment{definition}{
  breakable,
  colback=black!10,
  colframe=white,%
  width=\dimexpr\linewidth+10pt\relax,%
  enlarge left by=-5pt,%
  enlarge right by=-5pt,%
  boxsep=5pt,%
  boxrule=0pt,
  left=0pt,right=0pt,top=0pt,bottom=0pt,
  sharp corners,
  before skip=0.5\baselineskip, %
  after skip=0.5\baselineskip,  %
  fonttitle=\bfseries, %
  coltitle=black %
}
\tcolorboxenvironment{assumption}{
  breakable,
  colback=black!10,
  colframe=white,%
  width=\dimexpr\linewidth+10pt\relax,%
  enlarge left by=-5pt,%
  enlarge right by=-5pt,%
  boxsep=5pt,%
  boxrule=0pt,
  left=0pt,right=0pt,top=0pt,bottom=0pt,
  sharp corners,
  before skip=0.5\baselineskip, %
  after skip=0.5\baselineskip,  %
  fonttitle=\bfseries, %
  coltitle=black %
}

\tcolorboxenvironment{hypothesis}{
  breakable,
  colback=black!10,
  colframe=white,%
  width=\dimexpr\linewidth+10pt\relax,%
  enlarge left by=-5pt,%
  enlarge right by=-5pt,%
  boxsep=5pt,%
  boxrule=0pt,
  left=0pt,right=0pt,top=0pt,bottom=0pt,
  sharp corners,
  before skip=0.5\baselineskip, %
  after skip=0.5\baselineskip,  %
  fonttitle=\bfseries, %
  coltitle=black %
}

\tcolorboxenvironment{fact}{
  breakable,
  colback=black!10,
  colframe=white,%
  width=\dimexpr\linewidth+10pt\relax,%
  enlarge left by=-5pt,%
  enlarge right by=-5pt,%
  boxsep=5pt,%
  boxrule=0pt,
  left=0pt,right=0pt,top=0pt,bottom=0pt,
  sharp corners,
  before skip=0.5\baselineskip, %
  after skip=0.5\baselineskip,  %
  fonttitle=\bfseries, %
  coltitle=black %
}
\usepackage{cleveref}
\crefname{theorem}{Theorem}{Theorems}
\crefname{equation}{Equation}{Equation}
\crefname{algorithm}{Algorithm}{Algorithm}
\crefname{appendix}{Appendix}{Appendices}
\crefname{proposition}{Proposition}{Propositions}
\crefname{lemma}{Lemma}{Lemmas}
\crefname{corollary}{Corollary}{Corollaries}
\crefname{definition}{Definition}{Definitions}
\crefname{assumption}{Assumption}{Assumptions}
\crefname{remark}{Remark}{Remarks}
\crefname{problem}{Problem}{Problems}
\crefname{property}{Property}{property}
\crefname{section}{Section}{Sections}
\crefname{appendix}{Appendix}{Appendices}
\crefname{proof}{Proof}{Proofs}
\crefname{figure}{Figure}{Figures}
\crefname{table}{Table}{Tables}

\numberwithin{equation}{section}
\numberwithin{theorem}{section}
\numberwithin{proposition}{section}
\numberwithin{definition}{section}
\numberwithin{lemma}{section}
\numberwithin{assumption}{section}
\numberwithin{remark}{section}

\newcommand{\Entmax}{\mathop{\rm{Entmax15}}}
\usepackage{makecell}

\usepackage{tocloft}   %

\usepackage{titletoc}

\def\Snospace~{\S{}}

\newcommand{\sref}[2]{\hyperref[#2]{#1 \ref{#2}}}

\usepackage[table]{xcolor}
\title{Attention Sinks and Outliers in Attention Residuals
}

\definecolor{LightCyan}{rgb}{0.8, 0.9, 1}

\newcolumntype{b}{>{\columncolor{LightCyan}\hspace{0pt}
}c}

\author{%
  Haozheng Luo\textsuperscript{1,*,\textdagger} \quad
  Haoran Dai\textsuperscript{2,8,*,\textdagger} \quad
  Ching-Yuen Huang\textsuperscript{3,*,\textdagger} \quad
  Shaoyang Zhang\textsuperscript{4} \\[0.4em]
  \textbf{Xi Chen\textsuperscript{1} \quad
  Eric Hanchen Jiang\textsuperscript{5} \quad
  Yijiang Li\textsuperscript{6} \quad
  Chenghao Qiu\textsuperscript{7}} \\[0.4em]
  \textbf{Chenwei Xu\textsuperscript{1} \quad
  Zhenyu Pan\textsuperscript{1} \quad
  Haotian Zhang\textsuperscript{8} \quad
  Binghui Wang\textsuperscript{2,\textdagger} \quad
  Yan Chen\textsuperscript{1,\textdagger}} \\[0.7em]
  \textsuperscript{1}Northwestern University \quad
  \textsuperscript{2}Illinois Institute of Technology \quad
  \textsuperscript{3}Rutgers University \\
  \textsuperscript{4}University of Michigan--Ann Arbor \quad
  \textsuperscript{5}University of California, Los Angeles \\
  \textsuperscript{6}University of California, San Diego \quad
  \textsuperscript{7}Texas A\&M University \quad
  \textsuperscript{8}QuiverAI, Inc. \\[0.5em]
  {\small
  \textsuperscript{*}Equal contribution. \quad
  \textsuperscript{\textdagger}Correspondence:} \\[0.1em]
  {\small
  \href{mailto:hluo@u.northwestern.edu}{\texttt{hluo@u.northwestern.edu}},
  \href{mailto:hdai10@hawk.illinoistech.edu}{\texttt{hdai10@hawk.illinoistech.edu}},
  \href{mailto:chy.huang@rutgers.edu}{\texttt{chy.huang@rutgers.edu}}} \\
  {\small
  \href{mailto:bwang70@illinoistech.edu}{\texttt{bwang70@illinoistech.edu}},
  \href{mailto:ychen@northwestern.edu}{\texttt{ychen@northwestern.edu}}}
}

\setlist[itemize]{
  leftmargin=1em,   
  topsep    = \dimexpr 6pt-1em\relax  plus 1pt minus 1pt,
  itemsep   = 0pt plus 2pt,
  parsep    = 0pt plus 1pt,
  partopsep = 0pt}
\setlist[enumerate]{
  leftmargin=1.4em, 
  topsep    = \dimexpr 6pt-1em\relax  plus 1pt minus 1pt,
  itemsep   = 0pt plus 2pt,
  parsep    = 0pt plus 1pt,
  partopsep = 0pt}

\begin{document}
\setlength{\abovedisplayskip}{2pt}
\setlength{\abovedisplayshortskip}{2pt}
\setlength{\belowdisplayskip}{2pt}
\setlength{\belowdisplayshortskip}{2pt}

\maketitle
\begin{abstract}
We propose \textbf{OASIS}, an outlier- and sink-aware method that stabilizes dual-normalized attention-residual architectures through explicit null routing and token-to-depth null coupling. AttnResidual introduces an additional depth-wise normalization channel that improves inter-layer routing flexibility but can also amplify attention sinks, activation outliers, and low-bit quantization error. OASIS builds on explicit $\Softmax_1$-based null routes at both the token and depth levels and uses token-level null evidence to downweight depth branches exhibiting stronger null behavior. Theoretically, we characterize a conditional mechanism for sink-like attention concentration under dual normalization, offering insight into the low-bit sensitivity observed in AttnResidual. Experimentally, we compare \sys against five baselines on three language-model backbones and multiple language-modeling, reasoning, and long-context benchmarks and observe consistent improvements in both attention sink mitigation and post-quantization performance. Notably, relative to Vanilla AttnResidual, \sys reduces maximum infinity norm by \textbf{81.75\%} and average kurtosis by \textbf{95.90\%}, lowers W8A8 perplexity by \textbf{82.00\%}, and improves W4A4 GSM8K Pass@1 by \textbf{42.11\%} on average across LLaMA-3.2-1B, Qwen3-0.6B, and Phi-4. Code is available at: \url{https://github.com/robinzixuan/OASIS}.

\end{abstract}

\section{Introduction}
\label{sec:intro}
We propose \textbf{OASIS}, an outlier- and sink-aware method for improving the
quantization robustness of attention-residual architectures. Attention sinks and
activation outliers are two central obstacles to low-bit inference: concentrated
attention can repeatedly allocate mass to a small set of tokens, while unusually large
activation values expand the quantization range and increase quantization error. OASIS
addresses these effects through an explicit no-op mechanism based on the
$\Softmax_1$ function, following its use in \citet{luo2026frost,luo2025fast}.

Modern foundation models are primarily based on the Transformer architecture
introduced by \citet{vaswani2017attention}, which has achieved strong performance across
a broad range of tasks \cite{luo2026contrastive, he2025rhythm, he2024st, liu2024llava}.
To increase model capacity, \citet{team2026attention} propose AttnResidual, which
extends standard token-level attention with an additional depth-wise routing mechanism.
This design introduces a second Softmax normalization over residual branches, enabling
learned information flow across layers. This added flexibility also creates a second
simplex constraint. In the AttnResidual settings studied here, this constraint can
exacerbate instability under compression and long-context inference, where attention
outliers~\cite{hu2024outlier} and attention sinks~\cite{xiao2023efficient} already impair
quantization and attention allocation. In particular, when a token-level attention update
or residual-branch update should contribute little information, standard Softmax must
still assign its probability mass to real tokens and real branches, increasing routing
concentration and pressure toward outlier formation.

To address this limitation, OASIS introduces an explicit \emph{null route}, an
additional probability state representing no useful contribution, through $\Softmax_1$
at both the token and depth levels. It then computes, for each residual branch, the
average probability that its attention heads assign to the token-level null route.
Branches with above-average null probability are treated as less informative and receive
a lower depth-routing logit. Thus, OASIS connects token-level routing to depth-level
routing, allowing the model to allocate no-op mass explicitly while making branch
selection responsive to the information carried by each branch. The method modifies only
the routing mechanism and leaves the training objective unchanged.

\paragraph{Contributions.} Our main contributions are threefold:
\begin{itemize}[leftmargin=*]
    \item We show that the dual-normalization structure of AttnResidual can amplify activation outliers, attention sinks, and pressure toward depth collapse. We introduce \textbf{AoS}, which replaces the token- and depth-level $\Softmax$ operators with $\Softmax_1$ to provide explicit null routes at both levels.
    
    \item We propose \textbf{OASIS}, which extends AoS with token-to-depth null coupling. Token-level null probabilities are used to adjust depth routing, reducing the weight assigned to branches with stronger null behavior.
    
    \item Across LLaMA-3.2-1B, Qwen3-0.6B, and Phi-4, OASIS reduces average kurtosis by \textbf{95.90\%} and maximum infinity norm by \textbf{81.75\%} on average relative to Vanilla AttnResidual. It also lowers W8A8 perplexity by \textbf{82.00\%} and improves W4A4 GSM8K Pass@1 by \textbf{42.11\%} on average.
\end{itemize}

\begin{figure}
    \centering
    \vspace{-0.1in}
    \includegraphics[width=\linewidth]{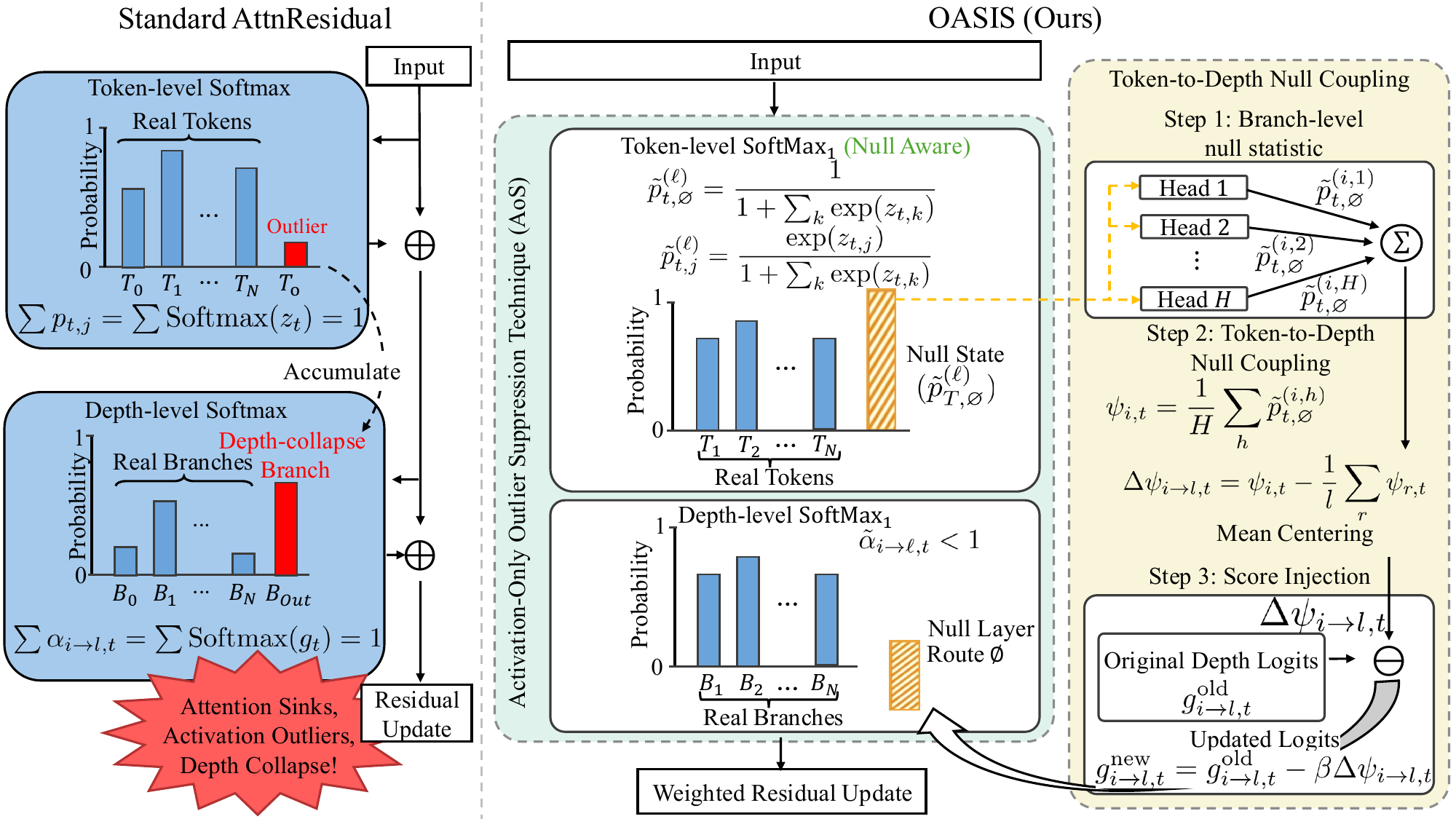}
    \caption{\textbf{Comparison of standard AttnResidual and OASIS.} OASIS adds null-aware token- and depth-level routing with token-to-depth null coupling, enabling explicit no-op allocation and reducing no-op outlier pressure, sink accumulation, and depth-collapse pressure.}
    \label{fig:fig1}
    \vspace{-0.2in}
\end{figure}

\section{Related Work}
\label{sec:background}
Transformer outliers are associated with structured attention, special-token concentration, and uneven token influence~\citep{kobayashi2020attention,clark2019what,kovaleva2019dark}, which hinder efficient and robust quantization~\citep{bondarenko2021understanding,shkolnik2020robust}. Standard softmax can further induce outliers when attention heads approximate partial or no-op updates~\citep{bondarenko2023quantizable}, while first-token concentration has been linked to large activations~\citep{kaul2025attentionactivation}. Existing remedies include Gumbel-Softmax~\citep{dadgarnia2026gsq}, gated attention~\citep{bondarenko2023quantizable}, and $\Softmax_1$-based designs~\citep{hu2024outlier}. By adding a constant to the denominator, $\Softmax_1$ permits probability mass to enter an implicit null state instead of forcing it onto real tokens~\citep{luo2025fast}. Appendix \ref{ap:related_works} shows additional related works.

\section{Failure Modes of Dual Normalization in AttnResidual}
\label{sec:observe}
In this section, we analyze the failure modes introduced by AttnResidual~\citep{gomez2026attnres} in LLaMA-3.2-1B relative to a vanilla baseline. We show that its dual-normalization design gives rise to two coupled pathologies: intensified attention sinks and amplified activation outliers. We then trace both effects to a common architectural cause and motivate \sys as a principled remedy. We compare AttnResidual with matched continued-pretraining controls in \Cref{tab:backbone} to separate architectural effects from continued-pretraining dynamics.
\begin{figure}[htp]
\vspace{-0.1in}
    \centering
    \begin{minipage}[c]{0.69\linewidth}
        \centering
        \includegraphics[width=\linewidth]{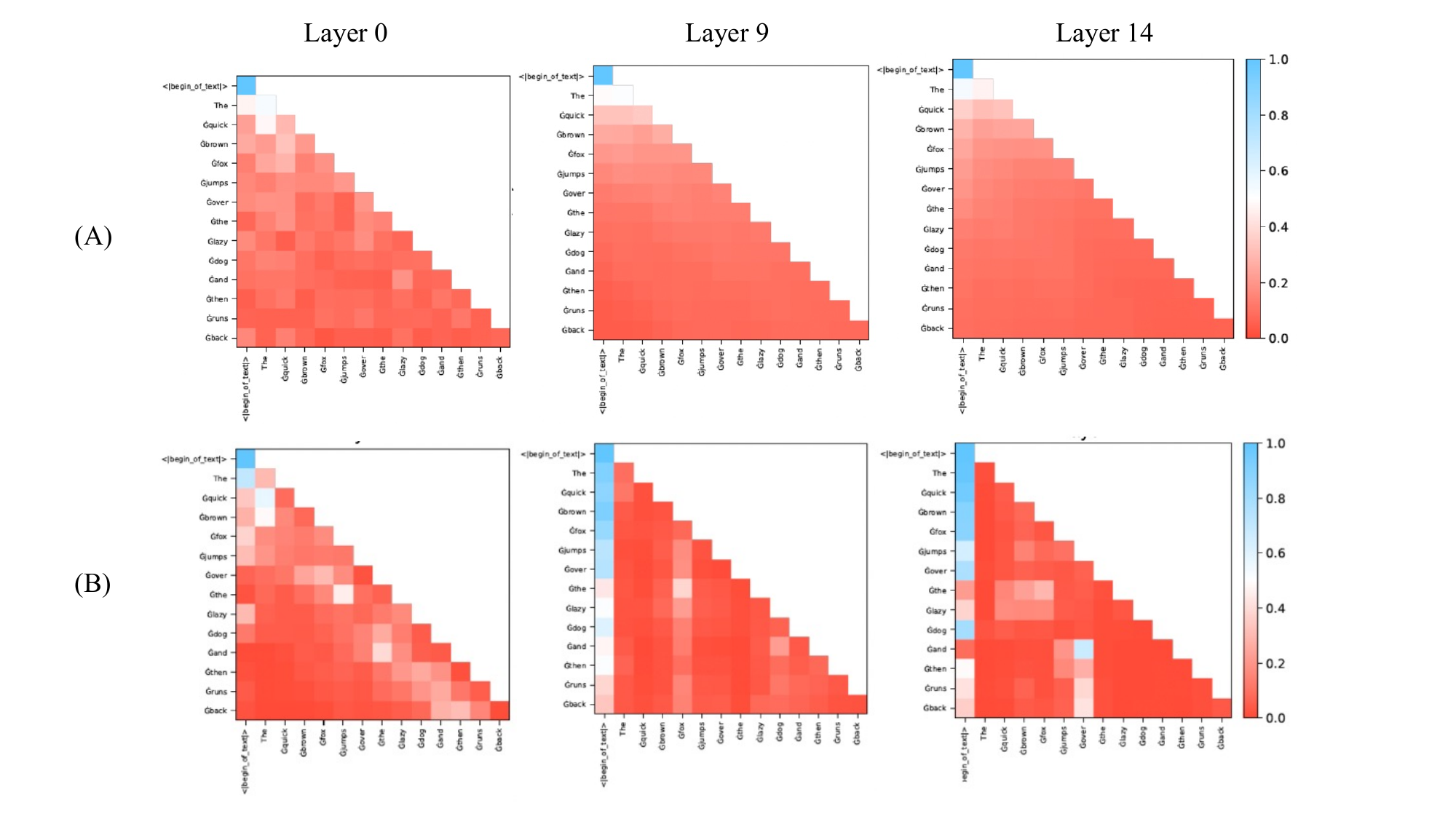}
    \end{minipage}
    \hfill
    \begin{minipage}[c]{0.28\linewidth}
        \captionof{figure}{
        \textbf{Attention sink visualization.}
        We visualize token-level attention maps for the
        (A) vanilla Transformer and (B) AttnResidual
        at layers 0, 9, and 14. The
        \texttt{<|begin\_of\_text|>} token acts as a persistent
        attention sink whose concentration increases with depth,
        particularly under the dual-normalization design of
        AttnResidual.
        }
        \label{fig:sink}
    \end{minipage}
\end{figure}

\begin{figure}[htp]
    \centering
    \includegraphics[width=\linewidth]{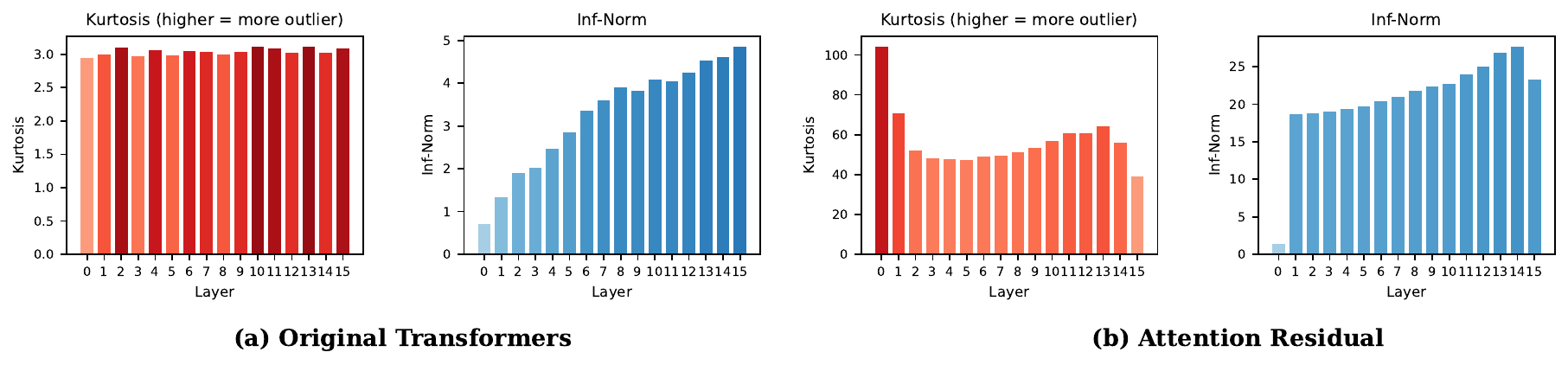}
    \vspace{-0.2in}
    \caption{\textbf{Outlier amplification in AttnResidual.}
Layer-wise hidden-state kurtosis and $\ell_\infty$ norm for the original Transformer and AttnResidual. AttnResidual exhibits substantially heavier-tailed activations and larger activation ranges across depth.}
    \label{fig:outlier}
\end{figure}

\paragraph{Intensified attention sinks.} We begin by examining token-level attention maps across depth. Although the \texttt{<|begin\_of\_text|>} token is already known to act as an attention sink in standard Transformers~\citep{xiao2023efficient}, we find that this effect becomes substantially more severe under AttnResidual (as shown in \cref{fig:sink}). In particular, the additional $\Softmax$ in the depth-routing branch reinforces sink accumulation rather than mitigating it.

\paragraph{Outlier amplification.} We next quantify the corresponding activation statistics across layers. Compared with the baseline, which maintains average
kurtosis around 3 and $\|x\|_\infty < 5$, AttnResidual produces substantially heavier-tailed hidden states (as shown in \cref{fig:outlier}). Compared with the original Transformers, Kurtosis in AttnResidual exceeds 100 in the first layer and remains at 50--70 thereafter, while $\|x\|_\infty$ reaches as high as 28. Such amplification is known to impair post-training quantization~\citep{hu2024outlier,llm.int.8} and destabilize long-context inference.

\paragraph{Mechanistic implications.} We find that both pathologies arise from the same architectural limitation: standard $\Softmax$ allocates all probability mass to active tokens and depth branches, leaving no null pathway. Consequently, under the separation conditions analyzed in \Cref{sec:theory}, producing a small attention contribution or a depth mixture close to a reference state can require concentrated routing, thereby amplifying sink attention and activation outliers.   Rather than suppressing these effects post hoc, \sys~(\Cref{sec:method}) introduces a $\Softmax_1$-based null-routing mechanism at both the token and depth levels, allowing the total contribution of real tokens and source branches to be attenuated without changing their relative allocations.

\section{Methodology}
\label{sec:method}
Motivated by the sink concentration and activation outliers observed in Section~\ref{sec:observe}, we first propose \emph{Activation-only Outlier Suppression} (AoS), which replaces token- and depth-level $\Softmax$ with $\Softmax_1$ to introduce explicit null states. We further propose \textbf{OASIS} (\cref{fig:fig1}), which aggregates token-level null probabilities across attention heads within each transformed source branch and uses the centered branch statistic to adjust depth-routing logits. Branches with above-average null mass receive a negative logit adjustment, coupling the two routing levels while retaining the standard training objective and residual-branch structure.

\subsection{Notation}
\label{sub:notation}
Following \citet{team2026attention}, the depth update at layer $\ell$ is:
$
h_{\ell}=\alpha_{0\to\ell}\cdot h_1+\sum_{i=1}^{\ell-1}\alpha_{i\to\ell}\cdot f_i(h_i),
$
with layer-mixing weights satisfying
$
\sum_{i=0}^{\ell-1}\alpha_{i\to\ell}=1.
$
Here, $f_i(\cdot)$ is the $i$-th layer transformation and $h_1$ is the initial residual state. For convenience, we denote the source states mixed by depth routing as $s_{0,t}:=h_t^{(1)}$ and $s_{i,t}:=f_i(h_t^{(i)})$ for $i\geq1$, so that $h_t^{(\ell)}=\sum_{i=0}^{\ell-1}\alpha_{i\to\ell,t}s_{i,t}$. For implementation, the router produces a depth-routing logit
$g_{i\to\ell,t}$ for each source branch $i$, target layer $\ell$, and token
$t$, and the mixing weights are
\[
\alpha_{i\to\ell,t}
=
\frac{\exp g_{i\to\ell,t}}
{\sum_{r=0}^{\ell-1}\exp g_{r\to\ell,t}},
\qquad i\in\{0,1,\ldots,\ell-1\}.
\]
where $g$ is a depth-routing \emph{score} (logit). Inside each $f_i$, token-level attention is
$$
p_{t,j}^{(\ell)}=\frac{\exp z_{t,j}^{(\ell)}}{\sum_{k=1}^{m}\exp z_{t,k}^{(\ell)}},
\qquad
A_t^{(\ell)}=\sum_{j=1}^{m}p_{t,j}^{(\ell)}v_j^{(\ell)}.
$$
Hence AttnResidual couples two normalizations: token $\Softmax$ ($p$) and depth $\Softmax$ ($\alpha$).
For theorem statements below, we write $\alpha_t^{(\ell)}$ as shorthand for depth weights when the index is implicit.

\subsection{Activation-Only Outlier Suppression (AoS)}
\label{sub:aos}

In autoregressive models, early tokens are visible to all subsequent positions, whereas later tokens are attended to only locally; this asymmetry can promote attention sinks \citep{xiao2023efficient}. Similar sink behavior also arises across depth: \citet{wong2025existence} identify persistent primary sinks in early layers and transient secondary sinks in middle layers. To mitigate sink-induced outliers at both routing levels, AoS replaces the standard Softmax with $\mathrm{Softmax}_1$ \citep{miller2021} in token-wise attention and depth routing. $\mathrm{Softmax}_1$ introduces an explicit null state that reduces the total mass assigned to real tokens or residual branches, while preserving their relative allocations. We formalize the token- and depth-level normalizations in Equations~\ref{eq:softmax1-token-trf} and~\ref{eq:softmax1-depth-trf}, respectively:
\begin{align}
\tilde p_{t,j}^{(\ell)}=\frac{\exp z_{t,j}^{(\ell)}}{1+\sum_{k=1}^{m}\exp z_{t,k}^{(\ell)}},
\quad
\tilde p_{t,\varnothing}^{(\ell)}=\frac{1}{1+\sum_{k=1}^{m}\exp z_{t,k}^{(\ell)}}, 
\label{eq:softmax1-token-trf}
\end{align}
\begin{align}
\tilde{\alpha}_{i \to \ell} 
= \frac{\exp g_{i \to \ell}}{1 + \sum_{k=0}^{\ell-1} \exp g_{k \to \ell}}, 
\quad
\tilde{\alpha}_{\varnothing \to \ell} 
= \frac{1}{1 + \sum_{k=0}^{\ell-1} \exp g_{k \to \ell}}.
\label{eq:softmax1-depth-trf}
\end{align}
Since the mass assigned to real tokens sums to less than one, token-level null routing attenuates the attention contribution by routing the remaining mass to a zero-valued null state. Analogously, replacing depth-$\Softmax$ with depth-$\Softmax_1$ attenuates the mixed contribution of real source branches. This depth-level null route reduces the probability budget available to real branches, and it does not by itself imply an identity mapping or preservation of a previous hidden state.

\subsection{Token-to-Depth Null Coupling (OASIS)}
While $\Softmax_1$ is null-aware, token- and depth-level routing remain decoupled,
so depth may still weight uninformative layers. 
We introduce \textbf{OASIS}, a coupling mechanism that propagates token-level null evidence into depth routing.
The key observation is that each source layer contributes a residual branch to depth routing. When most attention heads within such a branch route mass to the null space, the resulting branch output carries little informative content and should receive a lower depth weight.
\vspace{-0.5em}
\paragraph{Branch-level null statistic.}
For each transformed source branch
$i\in\{1,\ldots,\ell-1\}$, let
$\tilde{p}_{t,\emptyset}^{(i,h)}$ denote the null posterior of
attention head $h$ for token $t$. We define
$
\psi_{i,t}
:=
\frac{1}{H}
\sum_{h=1}^{H}
\tilde{p}_{t,\emptyset}^{(i,h)}
\in[0,1].
$
The initial-state branch $i=0$ has no attention heads and therefore
no token-level null posterior. We leave its routing logit unchanged
and center the null statistic only over transformed branches:
$$
\overline{\psi}_{\ell,t}
:=
\frac{1}{\ell-1}
\sum_{r=1}^{\ell-1}\psi_{r,t},
\qquad
\Delta\psi_{i\to\ell,t}
:=
\begin{cases}
\psi_{i,t}-\overline{\psi}_{\ell,t}, & i\ge1,\\
0, & i=0.
\end{cases}
$$
For $\ell=1$, no token-to-depth coupling is applied.

\paragraph{Score injection.}
Let $g^{\mathrm{old}}_{i\to\ell,\,t}$ denote the pre-normalization
depth routing logit produced by the router for source branch~$i$, token~$t$,
and target layer~$\ell$.
We subtract the centered null statistic, scaled by a non-negative coupling
strength~$\beta$:
\begin{align*}
  g^{\mathrm{new}}_{i\to\ell,\,t}
  = g^{\mathrm{old}}_{i\to\ell,\,t}
    - \beta\,\Delta\psi_{i\to\ell,\,t},
  \qquad \beta \ge 0.
\end{align*}
When $\Delta\psi_{i\to\ell,t} > 0$ (above-average null mass), the logit
decreases; when $\Delta\psi_{i\to\ell,t} < 0$ (below-average), it increases.
The adjusted logits are then normalized with depth-$\Softmax_1$:

\begin{align*}
  \tilde{\alpha}_{i\to\ell,\,t}
  = \frac{\exp\!\bigl(g^{\mathrm{new}}_{i\to\ell,\,t}\bigr)}
         {1 + \sum_{r=0}^{\ell-1}
          \exp\!\bigl(g^{\mathrm{new}}_{r\to\ell,\,t}\bigr)}.
\end{align*}

The resulting weight $\tilde{\alpha}_{i\to\ell,t}$ replaces the original
$\alpha_{i\to\ell}$ in the residual update of \cref{sub:notation}.

\section{Theoretical Analysis}
\label{sec:theory}
In this section, we analyze the routing constraints associated with dual normalization in AttnResidual \cite{team2026attention}, motivated by the routing mismatch discussed in \cref{sec:method}. Under the stated separation conditions, we first derive a lower
bound on the attention mass assigned to a low-norm token when a small attention output is required, and then examine the analogous concentration constraint for depth routing. We further characterize how depth mixing combines branchwise sink masses, including the case in which a single source branch dominates. Finally, we show that $\Softmax_1$, as used in AoS, provides an explicit null route for reducing contributions from real tokens and source branches, and characterize how OASIS uses token-level null evidence to adjust relative routing preferences among
transformed source branches. 

We use the notation introduced in \cref{sec:method}. In particular, AttnResidual depth mixing $(\alpha)$, token attention $(p)$, and the $\Softmax_1$ variables $(\tilde p)$ follow \cref{sec:method} without redefinition. All assumptions and additional definitions used in this section are provided in \cref{ap:assume}.

\begin{lemma}[Attention concentration under separated value geometry]
\label{lemma:1}
Under \cref{assum:token_sep,assum:1}, fix a token position $t\in\mathcal{T}_{\mathrm{noop}}$ and a branch $\ell$. If the attention output satisfies
\[
\|A_t^{(\ell)}\|=\Big\|\sum_{j=1}^{m}p_{t,j}^{(\ell)}v_j^{(\ell)}\Big\|\leq\delta,
\]
with $\delta<c_1$, then the attention weight on token $o$ satisfies
\[
p_{t,o}^{(\ell)}\geq\frac{c_1-\delta}{c_0+c_1}.
\]
\vspace{-0.6em}
\end{lemma}
\vspace{-1em}
\begin{proof}
See \cref{proof:LLaMA1} for a detailed proof.
\end{proof}
\vspace{-1.1em}
Lemma~\ref{lemma:1} shows that under the separation condition, a small attention output requires mass on the low-norm candidate, and the bound forces near-unit mass when both $c_0/c_1$ and $\delta/c_1$ approach zero.

\begin{definition}[Per-token pathology score]
\label{def:pathology_score}
\[
\mathcal{S}_t
=
\sum_{\ell}\alpha_t^{(\ell)}\sum_{j\in\mathcal{N}_t}p_{t,j}^{(\ell)}
+\lambda_1\max_{\ell,j}\alpha_t^{(\ell)}p_{t,j}^{(\ell)}
+\lambda_2\left(-H(\alpha_t)-\sum_{\ell}\alpha_t^{(\ell)}H(p_t^{(\ell)})\right).
\]
Here, $H(\cdot)$ denotes Shannon entropy. The three terms measure irrelevant-token leakage, peak routing concentration, and joint depth--token entropy collapse, respectively. We use $\mathcal{S}_t$ as a surrogate for these routing quantities, rather than a direct measure of activation outliers.
\vspace{-0.5em}
\end{definition}

\begin{definition}[Vanilla pathology score]
For vanilla attention (single token-$\Softmax$), we use
$$
\mathcal{S}_t^{\mathrm{Vanilla}}
=
\sum_{j\in\mathcal{N}_t}p_{t,j}
+\lambda_1\max_j p_{t,j}
+\lambda_2\big(-H(p_t)\big).
$$
\vspace{-0.6em}
\label{def:2}
\end{definition}

\begin{proposition}[Conditional aggregation of routing pathologies]
\label{proposition:attnresidual_pathology}
Under matched no-op tolerance, suppose that AttnResidual satisfies
\cref{assum:3}, i.e., its irrelevant-token leakage, peak routing concentration, and entropy-collapse terms are individually no smaller than those of vanilla attention. Then, for
$\lambda_1,\lambda_2\geq 0$,
\[
S_t^{\mathrm{AttnResidual}}
\geq
S_t^{\mathrm{Vanilla}}.
\]
\end{proposition}

\begin{theorem}[Conditional concentration in depth routing]
\label{thm:depth_collapse}
\label{thm:therm2}
Let the depth-composed deviation from a reference state be
\[
\Delta h_t^{(\ell)}=\sum_{i=0}^{\ell-1}\alpha_{i\to\ell,t},
u_{i,t},
\qquad
\alpha_{i\to\ell,t}\geq 0,
\qquad
\sum_{i=0}^{\ell-1}\alpha_{i\to\ell,t}=1.
\]
Suppose that the composed deviation is bounded by $\left\|\Delta h_t^{(\ell)}\right\|\leq\delta_t$, and that there exist a minimal-deviation branch $i_t^\ast$ and a unit vector $w_t^{(\ell)}$ satisfying
\[
\left\|u_{i_t^\ast,t}\right\|\leq b_0,
\qquad
\left\langle u_{i,t},w_t^{(\ell)}\right\rangle\geq b_1
\quad\text{for all }i\neq i_t^\ast,
\qquad
b_1>\delta_t.
\]
Then the depth-routing weights satisfy
\[
\alpha_{i_t^\ast\to\ell,t}\geq\frac{b_1-\delta_t}{b_0+b_1},
\qquad
\sum_{i\neq i_t^\ast}\alpha_{i\to\ell,t}\leq\frac{b_0+\delta_t}{b_0+b_1}.
\]
Consequently, in the high-separation regime
\[
\frac{b_0}{b_1}\to0,
\qquad
\frac{\delta_t}{b_1}\to0,
\]
the routing mass concentrates on the minimal-deviation branch:
\[
\alpha_{i_t^\ast\to\ell,t}\to1.
\]
\vspace{-0.5em}
\end{theorem}

\vspace{-0.6em}
\begin{proof}
See \cref{proof:therm2} for a detailed proof.
\end{proof}
\vspace{-1em}
Theorem \ref{thm:depth_collapse} provides a conditional concentration bound: it does not imply that depth routing universally collapses, but characterizes the regime in which a near-zero composed deviation forces most routing mass onto a minimal-deviation branch.

\begin{lemma}[Sink aggregation under depth routing]
\label{lemma:lemma2}
Let the branchwise and depth-aggregated sink masses be
\[
\sigma_t^{(\ell)}:=\sum_{j\in\mathcal{K}}p_{t,j}^{(\ell)},
\qquad
\Sigma_t:=\sum_{\ell}\alpha_t^{(\ell)}\sigma_t^{(\ell)}.
\]
Since the depth-routing weights lie on the simplex,
\[
\min_{\ell}\sigma_t^{(\ell)}
\leq
\Sigma_t
\leq
\max_{\ell}\sigma_t^{(\ell)}.
\]
If depth routing conditionally concentrates on a branch $\ell^\ast$, such that $\alpha_t^{(\ell^\ast)}\to1$, then
\[
\Sigma_t\to\sigma_t^{(\ell^\ast)}.
\]
Thus, depth routing forms a convex combination of branchwise sink masses and approaches the selected branch after concentration.
\end{lemma}

\begin{proof}
See \cref{proof:llemma2} for a detailed proof.
\end{proof}
\vspace{-1em}
Lemma \ref{lemma:lemma2} characterizes how existing branchwise sink mass is aggregated by depth routing; it does not by itself imply that depth routing creates or universally preserves an attention sink.

\begin{theorem}[$\Softmax_1$ enables attenuation of real-token and branch contributions]
\label{thm:therm3}
Introduce explicit null mass through $\Softmax_1$:
\[
\tilde p_{t,j}^{(\ell)}
=
\frac{\exp z_{t,j}^{(\ell)}}{1+\sum_k\exp z_{t,k}^{(\ell)}},
\qquad
\tilde p_{t,\varnothing}^{(\ell)}
=
\frac{1}{1+\sum_k\exp z_{t,k}^{(\ell)}}.
\]
Assume that $\|v_j^{(\ell)}\|\leq V_{\max}$ for every real token $j$ and
\[
B_{\max}:=\max_{0\leq i\leq\ell-1}\|s_{i,t}\|<\infty.
\]
Then $\Softmax_1$ attenuates real-token and branch contributions while preserving their relative allocations, without requiring concentration on any particular token or branch. Its effects on outliers, sinks, and depth routing are evaluated empirically.
\vspace{-0.3em}
\end{theorem}

\vspace{-0.5em}
\begin{proof}
See \cref{proof:therm3} for a detailed proof.
\end{proof}
\vspace{-1em}

\cref{thm:therm3} provides the theoretical basis for AoS: explicit null routes allow token- and depth-level contributions to be attenuated without forcing the remaining probability mass to concentrate on real tokens or source branches.
\vspace{-0.8em}

\section{Experimental Studies}
\label{sec:exp}We conduct a series of experiments to evaluate \sys for quantization robustness and its ability to mitigate attention sink, benchmarking on Qwen3~\cite{yang2025qwen3} and LLaMA~\cite{grattafiori2024LLaMA}. All experiments run with three independent random seeds, and we report mean and standard deviation for each metric.

\paragraph{Models.}
In our experiments, we use LLaMA-3.2~\cite{grattafiori2024LLaMA} and Qwen3~\cite{yang2025qwen3} as backbone models to evaluate quantization robustness and attention sink. Specifically, we continue training the LLaMA-3.2-1B\footnote{https://huggingface.co/meta-llama/Llama-3.2-1B}, Qwen3-0.6B\footnote{https://huggingface.co/Qwen/Qwen3-0.6B} and Phi-4\footnote{https://huggingface.co/microsoft/Phi-4-reasoning} checkpoints with the attention residual architecture, following \citet{gomez2026attnres}.

\paragraph{Data.}
Following the setup of \citet{hu2024outlier}, we use two real-world corpora for training: BookCorpus \cite{Zhu_2015_ICCV} and Wiki40B/en \cite{guo2020wiki}. To comprehensively evaluate the impact of Softmax variants on attention sinks and quantization robustness, we consider a diverse suite of downstream benchmarks. Specifically, we use WikiText-2 \cite{merity2016pointer} and The Pile to evaluate language-modeling performance and measure perplexity before and after quantization; GSM8K \cite{cobbe2021training}, MMLU \cite{hendrycks2021measuring}, TriviaQA \cite{joshi2017triviaqa}, and MATH500 \cite{hendrycks2021measuring} to assess reasoning and knowledge capabilities; and RULER \cite{hsieh2024ruler} and LongBench \cite{bai2024longbench} to evaluate long-context understanding. For accuracy-based benchmarks, higher values indicate better performance, whereas lower perplexity indicates better language-modeling and quantization performance.

\subsection{Outlier Efficiency of AoS}

\paragraph{Metrics.}
We report the \textit{maximum infinity norm} $||x||_{\infty}$ and the \textit{average kurtosis} of activation tensors $\mathbf{x}$ across all transformer layers as outlier metrics, averaging over all output components within the transformer stack. These statistics are strongly correlated with model quantizability and capture robustness to activation outliers~\cite{hu2024outlier,bondarenko2021understanding, shkolnik2020robust}. Prior work~\cite{llm.int.8, wei2022outlier, bondarenko2021understanding} has shown that the presence of outliers can substantially degrade model performance after quantization, so we evaluate each model both before and after quantization.
For pre-quantization performance, we report \textit{perplexity} in \textbf{FP16} (16-bit floating point). For post-quantization performance, we report the same metric under \textbf{W8A8}, and further include \textbf{W6A6} and \textbf{W4A4} results using OmniQuant~\cite{shao2024omniquant}. OmniQuant is a post-training quantization method for large language models that improves low-bit quantization by jointly optimizing weight and activation quantization with learnable transformations, without requiring full model retraining.

\begin{table*}[htp]
    \centering
    \caption{
\textbf{Outlier reduction and quantization performance.}
Comparison of $\mathtt{AoS}$ with Vanilla, OutEffHop, Clipped Softmax, and Gated Attention on $\mathtt{LLaMA\text{-}3.2\text{-}1B}$ , $\mathtt{Qwen3\text{-}0.6B}$ and $\mathtt{Phi\text{-}4}$.
We report average kurtosis, maximum activation magnitude $\|\mathbf{x}\|_{\infty}$, and perplexity under FP16, W8A8, and OmniQuant W6A6/W4A4.
Best results are in \textbf{bold}, and the second-best results are \underline{underlined}.
}
\renewcommand{\arraystretch}{0.9}
\label{tab:result1}
\resizebox{\textwidth}{!}{%
\begin{tabular}{cccccccc}
    \toprule
    \multirow{2}{*}[-2.5pt]{Model} &
    \multirow{2}{*}[-2.5pt]{Method} &
    \multirow{2}{*}[-2.5pt]{Avg. kurtosis} &
    \multirow{2}{*}[-2.5pt]{Max inf. norm} &
    \multirow{2}{*}[-2.5pt]{FP16} &
    \multirow{2}{*}[-2.5pt]{W8A8} &
    \multicolumn{2}{c}{OmniQuant} \\
    \cmidrule(lr){7-8}
    & & & & & &
    \raisebox{-1pt}{W6A6} &
    \raisebox{-1pt}{W4A4} \\
    \midrule

\multirow{9}{1em}{\rot{LLaMA-3.2-1B}}
    & Vanilla AttnResidual
        & 91.71 $\pm$ 0.08
        & 84.43 $\pm$ 0.15
        & \textbf{8.98} $\pm$ 0.01
        & 30.19 $\pm$ 0.06
        & 39.38 $\pm$ 0.01
        & 74.97 $\pm$ 0.02 \\
    & OutEffHop
        & \underline{3.08} $\pm$ 0.02
        & \underline{13.78} $\pm$ 0.08
        & 9.07 $\pm$ 0.01
        & \underline{10.42} $\pm$ 0.01
        & \underline{13.56} $\pm$ 0.00
        & \underline{26.18} $\pm$ 0.00 \\
    & AoS
        & \cellcolor{LightCyan} \textbf{3.03} $\pm$ 0.01
        & \cellcolor{LightCyan} \textbf{12.54} $\pm$ 0.07
        & \cellcolor{LightCyan} \underline{9.00} $\pm$ 0.01
        & \cellcolor{LightCyan} \textbf{9.36} $\pm$ 0.01
        & \cellcolor{LightCyan} \textbf{11.28} $\pm$ 0.00
        & \cellcolor{LightCyan} \textbf{22.63} $\pm$ 0.00 \\
    \cline{2-8}
    & Clipped Softmax
        & 3.30 $\pm$ 0.00
        & 14.24 $\pm$ 0.01
        & \underline{9.05} $\pm$ 0.00
        & 10.26 $\pm$ 0.00
        & 14.37 $\pm$ 0.00
        & 26.81 $\pm$ 0.00 \\
    & Clipped OutEffHop
        & \underline{2.40} $\pm$ 0.00
        & \underline{10.85} $\pm$ 0.01
        & 9.08 $\pm$ 0.00
        & \underline{10.19} $\pm$ 0.00
        & \underline{13.28} $\pm$ 0.00
        & \underline{25.62} $\pm$ 0.00 \\
    & Clipped AoS
        & \cellcolor{LightCyan} \textbf{2.36} $\pm$ 0.00
        & \cellcolor{LightCyan} \textbf{9.87} $\pm$ 0.01
        & \cellcolor{LightCyan} \textbf{9.01} $\pm$ 0.00
        & \cellcolor{LightCyan} \textbf{9.34} $\pm$ 0.00
        & \cellcolor{LightCyan} \textbf{11.25} $\pm$ 0.00
        & \cellcolor{LightCyan} \textbf{22.57} $\pm$ 0.00 \\
    \cline{2-8}
    & Gated Attention
        & 4.39 $\pm$ 0.02
        & 14.89 $\pm$ 0.01
        & \textbf{8.26} $\pm$ 0.00
        & 9.21 $\pm$ 0.01
        & 13.10 $\pm$ 0.02
        & 24.27 $\pm$ 0.03 \\
    & Gated OutEffHop
        & \underline{2.14} $\pm$ 0.01
        & \underline{8.33} $\pm$ 0.00
        & 8.75 $\pm$ 0.00
        & \underline{9.20} $\pm$ 0.03
        & \underline{12.08} $\pm$ 0.04
        & \underline{23.23} $\pm$ 0.08 \\
    & Gated AoS
        & \cellcolor{LightCyan} \textbf{2.11} $\pm$ 0.01
        & \cellcolor{LightCyan} \textbf{8.03} $\pm$ 0.00
        & \cellcolor{LightCyan} \underline{8.68} $\pm$ 0.00
        & \cellcolor{LightCyan} \textbf{8.88} $\pm$ 0.03
        & \cellcolor{LightCyan} \textbf{10.70} $\pm$ 0.03
        & \cellcolor{LightCyan} \textbf{21.45} $\pm$ 0.07 \\
\midrule
\multirow{9}{1em}{\rot{Qwen3-0.6B}}
    & Vanilla AttnResidual
        & 93.04 $\pm$ 0.10
        & 74.78 $\pm$ 0.13
        & 10.57 $\pm$ 0.01
        & 36.64 $\pm$ 0.08
        & 44.78 $\pm$ 0.05
        & 50.82 $\pm$ 0.59 \\
    & OutEffHop
        & \underline{3.04} $\pm$ 0.02
        & \underline{14.16} $\pm$ 0.06
        & \underline{9.99} $\pm$ 0.01
        & \underline{10.60} $\pm$ 0.01
        & \underline{13.36} $\pm$ 0.03
        & \underline{32.36} $\pm$ 0.37 \\
    & AoS
        & \cellcolor{LightCyan} \textbf{3.03} $\pm$ 0.02
        & \cellcolor{LightCyan} \textbf{13.90} $\pm$ 0.05
        & \cellcolor{LightCyan} \textbf{9.35} $\pm$ 0.01
        & \cellcolor{LightCyan} \textbf{9.62} $\pm$ 0.01
        & \cellcolor{LightCyan} \textbf{10.48} $\pm$ 0.03
        & \cellcolor{LightCyan} \textbf{17.52} $\pm$ 0.39 \\
    \cline{2-8}
    & Clipped Softmax
        & 3.03 $\pm$ 0.00
        & 14.16 $\pm$ 0.01
        & 10.37 $\pm$ 0.00
        & 21.38 $\pm$ 0.01
        & 28.74 $\pm$ 0.06
        & 39.89 $\pm$ 0.60 \\
    & Clipped OutEffHop
        & \underline{3.02} $\pm$ 0.00
        & \underline{12.69} $\pm$ 0.01
        & \underline{9.87} $\pm$ 0.00
        & \underline{10.44} $\pm$ 0.00
        & \underline{11.39} $\pm$ 0.03
        & \underline{30.74} $\pm$ 0.36 \\
    & Clipped AoS
        & \cellcolor{LightCyan} \textbf{3.02} $\pm$ 0.00
        & \cellcolor{LightCyan} \textbf{12.45} $\pm$ 0.01
        & \cellcolor{LightCyan} \textbf{9.23} $\pm$ 0.00
        & \cellcolor{LightCyan} \textbf{9.47} $\pm$ 0.00
        & \cellcolor{LightCyan} \textbf{10.33} $\pm$ 0.03
        & \cellcolor{LightCyan} \textbf{16.95} $\pm$ 0.36 \\
    \cline{2-8}
    & Gated Attention
        & 3.03 $\pm$ 0.00
        & 14.36 $\pm$ 0.02
        & 10.56 $\pm$ 0.00
        & 19.26 $\pm$ 0.01
        & 23.43 $\pm$ 0.05
        & 28.78 $\pm$ 0.53 \\
    & Gated OutEffHop
        & \underline{3.03} $\pm$ 0.00
        & \underline{13.81} $\pm$ 0.01
        & \underline{10.00} $\pm$ 0.00
        & \underline{10.57} $\pm$ 0.00
        & \underline{11.52} $\pm$ 0.00
        & \underline{21.26} $\pm$ 0.02 \\
    & Gated AoS
        & \cellcolor{LightCyan} \textbf{3.03} $\pm$ 0.00
        & \cellcolor{LightCyan} \textbf{13.55} $\pm$ 0.01
        & \cellcolor{LightCyan} \textbf{9.35} $\pm$ 0.00
        & \cellcolor{LightCyan} \textbf{9.59} $\pm$ 0.00
        & \cellcolor{LightCyan} \textbf{10.46} $\pm$ 0.00
        & \cellcolor{LightCyan} \textbf{15.43} $\pm$ 0.30 \\
\midrule

\multirow{9}{1em}{\rot{Phi-4}}
    & Vanilla & 38.83 $\pm$ 0.02 & 37.96 $\pm$ 0.02 & 11.21 $\pm$ 0.00 & 1176.73 $\pm$ 0.01 & 1611.07 $\pm$ 0.01 & 2724.05 $\pm$ 0.02 \\
    & OutEffHop & \underline{3.95} $\pm$ 0.00 & \underline{18.70} $\pm$ 0.01 & \textbf{10.93} $\pm$ 0.00 & \underline{16.98} $\pm$ 0.00 & \underline{37.75} $\pm$ 0.00 & \underline{160.70} $\pm$ 0.00 \\
    & AoS & \cellcolor{LightCyan} \textbf{3.85} $\pm$ 0.00 & \cellcolor{LightCyan} \textbf{9.11} $\pm$ 0.01 & \cellcolor{LightCyan} \underline{11.09} $\pm$ 0.00 & \cellcolor{LightCyan} \textbf{12.18} $\pm$ 0.00 & \cellcolor{LightCyan} \textbf{16.85} $\pm$ 0.00 & \cellcolor{LightCyan} \textbf{48.26} $\pm$ 0.00 \\
    \cline{2-8}
    & Clipped Softmax & 4.76 $\pm$ 0.01 & 24.53 $\pm$ 0.01 & 11.43 $\pm$ 0.00 & 23.08 $\pm$ 0.01 & 46.64 $\pm$ 0.03 & 112.18 $\pm$ 0.30 \\
    & Clipped OutEffHop & \underline{3.51} $\pm$ 0.01 & \underline{15.75} $\pm$ 0.01 & \textbf{10.86} $\pm$ 0.00 & \underline{16.66} $\pm$ 0.00 & \underline{29.59} $\pm$ 0.02 & \underline{90.74} $\pm$ 0.36 \\
    & Clipped AoS & \cellcolor{LightCyan} \textbf{3.42} $\pm$ 0.00 & \cellcolor{LightCyan} \textbf{7.69} $\pm$ 0.01 & \cellcolor{LightCyan} \underline{11.02} $\pm$ 0.00 & \cellcolor{LightCyan} \textbf{12.01} $\pm$ 0.00 & \cellcolor{LightCyan} \textbf{16.70} $\pm$ 0.01 & \cellcolor{LightCyan} \textbf{37.62} $\pm$ 0.38 \\
    \cline{2-8}
    & Gated Attention & 4.05 $\pm$ 0.01 & 24.96 $\pm$ 0.02 & 11.07 $\pm$ 0.00 & 20.76 $\pm$ 0.01 & 42.82 $\pm$ 0.03 & 102.45 $\pm$ 0.28 \\
    & Gated OutEffHop & \underline{3.34} $\pm$ 0.01 & \underline{14.81} $\pm$ 0.01 & \textbf{10.75} $\pm$ 0.00 & \underline{15.97} $\pm$ 0.03 & \underline{37.16} $\pm$ 0.02 & \underline{46.69} $\pm$ 0.05 \\
    & Gated AoS & \cellcolor{LightCyan} \textbf{3.26} $\pm$ 0.01 & \cellcolor{LightCyan} \textbf{7.44} $\pm$ 0.01 & \cellcolor{LightCyan} \underline{10.90} $\pm$ 0.00 & \cellcolor{LightCyan} \textbf{11.36} $\pm$ 0.03 & \cellcolor{LightCyan} \textbf{14.34} $\pm$ 0.02 & \cellcolor{LightCyan} \textbf{31.89} $\pm$ 0.28 \\
\bottomrule
    \end{tabular}%
    }
\end{table*}
\begin{table*}[hpt]
    \centering
    \scriptsize
\caption{\textbf{OASIS performance across common language models.} We conduct experiments on $\mathtt{LLaMA\text{-}3.2\text{-}1B}$ , $\mathtt{Qwen3\text{-}0.6B}$ and $\mathtt{Phi\text{-}4}$ to evaluate activation stability and quantization robustness. Compared with Vanilla AttnResidual, OASIS reduces average kurtosis by \textbf{95.90\%}, maximum infinity norm by \textbf{81.75\%}, and W8A8 perplexity by \textbf{82.00\%}, while improving W4A4 GSM8K accuracy by \textbf{42.11\%} on average. Best results are in \textbf{bold} and second-best results are \underline{underlined}.}
\renewcommand{\arraystretch}{0.85}
\vspace{-0.1in}
    \label{tab:result2}
    \resizebox{\textwidth}{!}{%
    \begin{tabular}{cccccccc}
    \toprule
    \multirow{2}{*}[-2.5pt]{Model} &
    \multirow{2}{*}[-2.5pt]{Method} &
    \multirow{2}{*}[-2.5pt]{Avg. kurtosis} &
    \multirow{2}{*}[-2.5pt]{Max inf. norm} &
    \multicolumn{2}{c}{Perplexity$\downarrow$} &
    \multicolumn{2}{c}{GSM8K$\uparrow$} \\
    \cmidrule(lr){5-6}\cmidrule(lr){7-8}
    & & & &
    \raisebox{-1pt}{FP16} &
    \raisebox{-1pt}{W8A8} &
    \raisebox{-1pt}{FP16} &
    \raisebox{-1pt}{W4A4} \\
    \midrule

       \multirow{4}{1em}{\rot{\makecell{LLaMA\\-3.2-1B}}}
    & Vanilla
       & 91.71 $\pm$ 0.08
        & 84.43 $\pm$ 0.15
        & \underline{8.98} $\pm$ 0.01
        & 30.19 $\pm$ 0.06
         & \underline{13.64} $\pm$ 0.17
        & 12.97 $\pm$ 0.24 \\
    & OutEffHop
        & 3.08 $\pm$ 0.02
        & 13.78 $\pm$ 0.08
        & 9.07 $\pm$ 0.01
        & 10.42 $\pm$ 0.01
         & 13.58 $\pm$ 0.15
        & 13.16 $\pm$ 0.21\\
    & AoS
        & \cellcolor{LightCyan} \underline{3.03} $\pm$ 0.01
        & \cellcolor{LightCyan} \underline{12.54} $\pm$ 0.07
        & \cellcolor{LightCyan} 9.00 $\pm$ 0.01
        & \cellcolor{LightCyan} \underline{9.36} $\pm$ 0.01
        & \cellcolor{LightCyan} 13.62 $\pm$ 0.13
        & \cellcolor{LightCyan} \underline{13.34} $\pm$ 0.18 \\
    & OASIS
        & \cellcolor{LightCyan} \textbf{3.02} $\pm$ 0.02
        & \cellcolor{LightCyan} \textbf{12.25} $\pm$ 0.07
        & \cellcolor{LightCyan} \textbf{8.01} $\pm$ 0.01
        & \cellcolor{LightCyan} \textbf{8.16} $\pm$ 0.01
        & \cellcolor{LightCyan} \textbf{13.75} $\pm$ 0.12
        & \cellcolor{LightCyan} \textbf{13.52} $\pm$ 0.17 \\
\midrule

\multirow{4}{1em}{\rot{\makecell{Qwen3\\-0.6B}}}
    & Vanilla
         & 93.04 $\pm$ 0.10
        & 74.78 $\pm$ 0.13
        & 10.57 $\pm$ 0.01
        & 36.64 $\pm$ 0.08
        & 62.58 $\pm$ 0.24
        & 45.67 $\pm$ 0.86 \\
    & OutEffHop
        & 3.04 $\pm$ 0.02
        & 14.16 $\pm$ 0.06
        & 9.99 $\pm$ 0.01
        & 10.60 $\pm$ 0.01
        & 62.59 $\pm$ 0.19
        & 52.36 $\pm$ 0.74\\
    & AoS
        & \cellcolor{LightCyan} \underline{3.03} $\pm$ 0.02
        & \cellcolor{LightCyan} \underline{13.90} $\pm$ 0.05
        & \cellcolor{LightCyan} \underline{9.35} $\pm$ 0.01
        & \cellcolor{LightCyan} \underline{9.62} $\pm$ 0.01
        & \cellcolor{LightCyan} \underline{62.65} $\pm$ 0.17
        & \cellcolor{LightCyan} \underline{54.21} $\pm$ 0.66 \\
    & OASIS
        & \cellcolor{LightCyan} \textbf{2.95} $\pm$ 0.02
        & \cellcolor{LightCyan} \textbf{13.00} $\pm$ 0.05
        & \cellcolor{LightCyan} \textbf{9.20} $\pm$ 0.01
        & \cellcolor{LightCyan} \textbf{9.52} $\pm$ 0.01
        & \cellcolor{LightCyan} \textbf{62.68} $\pm$ 0.15
        & \cellcolor{LightCyan} \textbf{58.88} $\pm$ 0.59 \\
\midrule

\multirow{4}{1em}{\rot{\makecell{Phi-4}}}
    & Vanilla
        & 38.83 $\pm$ 0.02
        & 37.96 $\pm$ 0.02
        & 11.21 $\pm$ 0.00
        & 1176.73 $\pm$ 0.01
        & 93.78 $\pm$ 0.17
        & 43.16 $\pm$ 0.25 \\
    & OutEffHop
        & 3.95 $\pm$ 0.00
        & 18.70 $\pm$ 0.01
        & \textbf{10.93} $\pm$ 0.00
        & 16.98 $\pm$ 0.00
        & \underline{93.91} $\pm$ 0.15
        & 71.42 $\pm$ 0.23 \\
    & AoS
        & \cellcolor{LightCyan} \underline{3.85} $\pm$ 0.00
        & \cellcolor{LightCyan} \underline{9.11} $\pm$ 0.01
        & \cellcolor{LightCyan} 11.09 $\pm$ 0.00
        & \cellcolor{LightCyan} \underline{12.18} $\pm$ 0.00
        & \cellcolor{LightCyan} 93.84 $\pm$ 0.16
        & \cellcolor{LightCyan} \underline{76.85} $\pm$ 0.21 \\
    & OASIS
        & \cellcolor{LightCyan} \textbf{2.27} $\pm$ 0.00
        & \cellcolor{LightCyan} \textbf{8.68} $\pm$ 0.01
        & \cellcolor{LightCyan} \underline{10.94} $\pm$ 0.01
        & \cellcolor{LightCyan} \textbf{11.72} $\pm$ 0.00
        & \cellcolor{LightCyan} \textbf{94.02} $\pm$ 0.15
        & \cellcolor{LightCyan} \textbf{83.37} $\pm$ 0.21 \\
\bottomrule
    \end{tabular}%
    }
 \vspace{-0.1in}
\end{table*}

\paragraph{Results.}
In \cref{tab:result1}, AoS consistently outperforms the baselines. Unlike OutEffHop, which modifies only the token-level $\Softmax$, AoS normalizes both token- and residual-level pathways, leading to more consistent outlier suppression and stronger post-quantization robustness. Averaged across the two backbones in the plain attention-residual setting, AoS reduces average kurtosis and maximum infinity norm relative to Vanilla by \textbf{96.72\%} and \textbf{83.40\%}, respectively, and lowers perplexity under W8A8, W6A6, and W4A4 by \textbf{71.60\%}, \textbf{72.36\%}, and \textbf{68.42\%}. These gains persist in clipped and gated variants, where the corresponding reductions are \textbf{10.95\%} and \textbf{15.40\%} for the outlier metrics, and \textbf{51.09\%}, \textbf{60.78\%}, and \textbf{57.79\%} for perplexity, showing that AoS complements clipping- and gating-based designs.

\begin{figure}[htp]
    \centering
    \vspace{-0.2in}    \includegraphics[width=\linewidth]{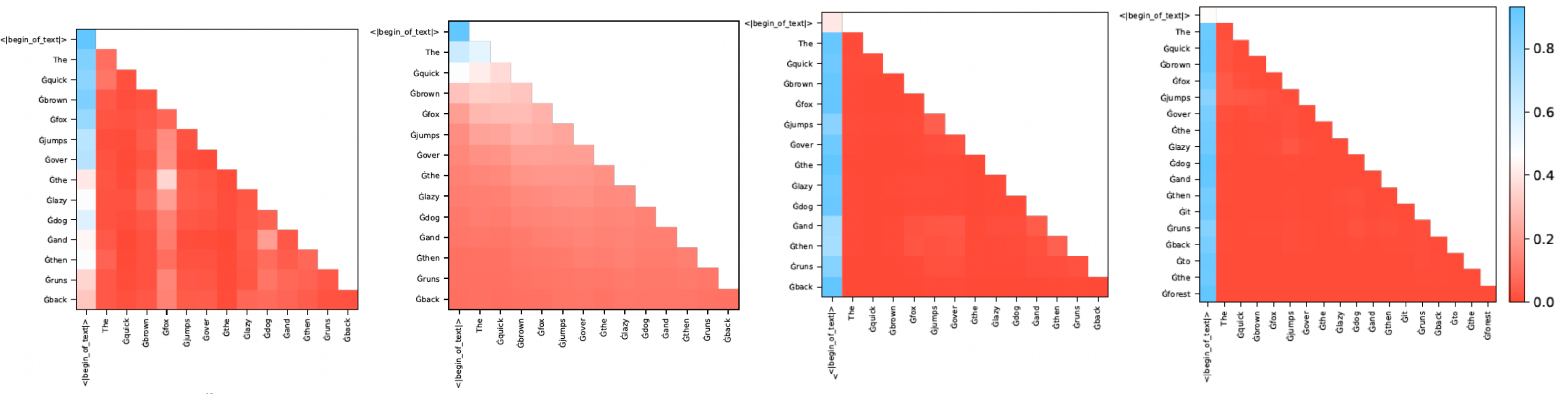}
    \vspace{-0.2in}
    \caption{\textbf{Attention sink mitigation in representative attention maps.} We visualize representative head-0 token-level attention maps on a short causal prompt for trained attention-residual variants. Strong vertical concentration on the leading \texttt{<|begin\_of\_text|>} token indicates an attention sink. Relative to standard variants, null-aware routing weakens both the dominant first-token sink and secondary sink bands, yielding a smoother and more local causal attention pattern.}
    \label{fig:sink2}
    \vspace{-0.1in}
\end{figure}

\subsection{Component Ablation of Null Routing and Token-to-Depth Coupling}
\label{sec:abla_component}

To isolate the contributions of token-level null routing, depth-level null routing, and token-to-depth coupling, we evaluate their combinations on Qwen3-0.6B under W4A4 quantization. As shown in \cref{tab:ablation_components}, token-only and depth-only routing improve GSM8K accuracy from \textbf{45.67\%} to \textbf{52.36\%} and \textbf{48.90\%}, respectively. Combining the two independent null routes yields \textbf{54.21\%}, while adding token-to-depth coupling further improves accuracy to \textbf{58.88\%}. These results show that coupling contributes beyond the two independent null routes.

\subsection{Outlier Efficiency of OASIS}

\paragraph{Setup.} To evaluate model capacity in a standard reasoning setting, we follow \citet{pan2025chainofaction} and benchmark all models on GSM8K with chain-of-thought prompting. We report Pass@1 as the evaluation metric and additionally evaluate perplexity under FP16 and W8A8, as well as Pass@1 under FP16 and W4A4 with OmniQuant.

\paragraph{Baselines.} We compare three representative methods spanning the main design paradigms for improving quantization robustness and mitigating attention sink: (1) \textbf{Clip Attention}~\cite{bondarenko2023quantizable}, which clips $\Softmax$ logits to limit the extreme attention scores that drive activation outliers; (2) \textbf{Gated Attention}~\cite{bondarenko2023quantizable}, which introduces an explicit gating mechanism that allows attention heads to perform partial or near no-op updates without inducing large $\Softmax$ magnitudes; and (3) \textbf{OutEffHop}~\cite{hu2024outlier}, which replaces standard attention with an outlier-efficient Hopfield-style retrieval mechanism and provides a principled $\Softmax_1$-based alternative for suppressing outlier formation. This set of baselines covers clipping-based, gating-based, and $\Softmax_1$-based approaches, providing a broad comparison across prior strategies for outlier control and quantization-friendly attention design.

\paragraph{Results.}
As shown in \cref{tab:result2}, OASIS consistently improves numerical stability and robustness under quantization on both models. On LLaMA-3.2-1B and Qwen3-0.6B, it lowers mean kurtosis and maximum infinity norm to \textbf{2.987} and \textbf{12.624}, and this translates into a mean W8A8 perplexity of \textbf{8.840} and a mean W4A4 GSM8K Pass@1 of \textbf{36.20\%}, corresponding to a \textbf{73.55\%} reduction and a \textbf{23.47\%} gain relative to Vanilla, respectively; the corresponding FP16 averages are \textbf{8.604} and \textbf{38.22\%}. Relative to AoS, OASIS further reduces mean W8A8 perplexity by \textbf{6.85\%} and improves mean W4A4 GSM8K Pass@1 by \textbf{7.18\%}.

\subsection{Improvement in Attention Sinks}
We evaluate sink behavior directly by visualizing token-level attention maps on a short causal prompt containing the leading \texttt{<|begin\_of\_text|>} token, which is known to induce sink formation~\cite{xiao2023efficient}. Using the trained models , we inspect the representative attention head from layer 14 and compare the resulting attention patterns qualitatively. In these maps, stronger vertical concentration on the initial token indicates a more severe sink, whereas a smoother lower-triangular pattern reflects healthier local attention.
As shown in \cref{fig:sink2}, standard attention-residual variants exhibit a pronounced sink on the initial token and, in some cases, develop secondary vertical bands over early tokens, indicating that attention mass remains concentrated on a small subset of positions. AoS partially alleviates this effect, but the bias toward early tokens remains visible. OASIS yields the cleanest pattern: attention is redistributed away from \texttt{<|begin\_of\_text|>}, secondary sinks are suppressed, and the maps recover a smoother causal decay dominated by nearby content tokens. These qualitative changes are consistent with the reductions in outlier metrics and support the central claim that null-aware routing mitigates the structural cause of sink accumulation rather than merely masking its downstream effects.

We also evaluate models trained with a maximum sequence length of 12K on RULER and LongBench. As shown in \cref{tab:long}, Vanilla AttnResidual substantially degrades long-context performance, whereas OASIS largely recovers the original backbone's performance on both benchmarks. On RULER, OASIS improves over Vanilla AttnResidual by 0.221 and 0.240 points for LLaMA-3.2-1B and Qwen3-0.6B, respectively, recovering \textbf{77.3\% and 80.3\% }of the performance lost relative to the original backbone. On LongBench, OASIS improves by 0.098 and 0.056 points, recovering \textbf{89.1\% and 91.8\%} of the corresponding degradation. These results indicate that OASIS stabilizes inter-layer attention routing without sacrificing the benefits of AttnResidual.

\subsection{Ablation Studies}

\begin{table*}[t]
\centering
{\begin{minipage}{\dimexpr\textwidth-2\fboxsep}
\vspace{-0.2in}
\caption{\textbf{Long-context evaluation on RULER and LongBench at 12K tokens.} We conduct experiments on $\mathtt{LLaMA\text{-}3.2\text{-}1B}$ and $\mathtt{Qwen3\text{-}0.6B}$ to evaluate long-context performance. OASIS consistently outperforms Vanilla AttnResidual and all sink-oriented baselines, improving RULER/LongBench by \textbf{+0.221/+0.098} and \textbf{+0.240/+0.056} over Vanilla AttnResidual, respectively. Variances are omitted as they are consistently $\leq 0.2\%$. Best results are shown in \textbf{bold}. }
\vspace{-0.1in}
\label{tab:long}
\resizebox{\linewidth}{!}{%
\begin{tabular}{llcccc}
\toprule
\multirow{2}{*}{Method} & \multirow{2}{*}{Role of absorbed mass} & \multicolumn{2}{c}{LLaMA-3.2-1B} & \multicolumn{2}{c}{Qwen3-0.6B} \\
\cmidrule(lr){3-4}\cmidrule(lr){5-6}
& & RULER $\uparrow$ & LongBench $\uparrow$ & RULER $\uparrow$ & LongBench $\uparrow$ \\
\midrule
Original backbone & & 0.786 & 0.225 & 0.891 & 0.194 \\
Vanilla AttnResidual &  & 0.500 & 0.115 & 0.592 & 0.133 \\
\midrule
S1: Learnable sink logit & Absorbed but not propagated to depth routing & 0.597 & 0.131 & 0.811 & 0.184 \\
S2: Prefix sink token & Contributes a learned value to the attention output & 0.586 & 0.128 & 0.800 & 0.181 \\
StreamingLLM & Retained in the KV cache but not propagated to depth routing & 0.578 & 0.124 & 0.792 & 0.177 \\
\midrule
AttnResidual + AoS & Not propagated across routing levels & 0.567 & 0.121 & 0.781 & 0.174 \\
AttnResidual + OASIS & Propagated to depth routing & \cellcolor{LightCyan}\textbf{0.721} & \cellcolor{LightCyan}\textbf{0.213} & \cellcolor{LightCyan}\textbf{0.832} & \cellcolor{LightCyan}\textbf{0.189} \\
\bottomrule
\end{tabular}}
\end{minipage}}
\end{table*}

\begin{table*}[t]
\centering
\caption{\textbf{Generalization across language modeling and downstream tasks.} We conduct experiments on Pile, MMLU, TriviaQA, and MATH500 using $\mathtt{LLaMA\text{-}3.2\text{-}1B}$ and $\mathtt{Qwen3\text{-}0.6B}$ to evaluate generalization beyond the primary benchmarks. OASIS consistently outperforms Vanilla AttnResidual and all competing variants under both FP16 and low-bit quantization, reducing W8A8 perplexity from \textbf{38.24 to 10.17} and \textbf{44.52 to 11.58}, while improving W4A4 accuracy across all downstream tasks. Lower perplexity and higher accuracy are better. Best results are shown in \textbf{bold}. }
\renewcommand{\arraystretch}{1.2}
\vspace{-0.1in}
\label{tab:generalization}
\resizebox{\textwidth}{!}{
\begin{tabular}{llcccccccc}
\toprule
& & \multicolumn{2}{c}{\textbf{The Pile} $\downarrow$}
& \multicolumn{2}{c}{\textbf{MMLU} $\uparrow$}
& \multicolumn{2}{c}{\textbf{TriviaQA} $\uparrow$}
& \multicolumn{2}{c}{\textbf{MATH500} $\uparrow$} \\
\cmidrule(lr){3-4}\cmidrule(lr){5-6}\cmidrule(lr){7-8}\cmidrule(lr){9-10}
\textbf{Model} & \textbf{Method}
& \textbf{FP16} & \textbf{W8A8}
& \textbf{FP16} & \textbf{W4A4}
& \textbf{FP16} & \textbf{W4A4}
& \textbf{FP16} & \textbf{W4A4} \\
\midrule

\multirow{5}{*}{LLaMA-3.2-1B}
& Original backbone
& 11.32 $\pm$ 0.01 & 11.79 $\pm$ 0.01
& 31.30 $\pm$ 0.22 & 29.01 $\pm$ 0.33
& 28.14 $\pm$ 0.30 & 24.40 $\pm$ 0.43
& 4.16 $\pm$ 0.44 & 2.59 $\pm$ 0.42 \\

& Vanilla AttnResidual
& 11.08 $\pm$ 0.01 & 38.24 $\pm$ 0.08
& 31.50 $\pm$ 0.21 & 26.42 $\pm$ 0.32
& 27.96 $\pm$ 0.32 & 19.13 $\pm$ 0.51
& 4.40 $\pm$ 0.43 & 2.81 $\pm$ 0.44 \\

& AttnResidual + OutEffHop
& 11.15 $\pm$ 0.01 & 12.87 $\pm$ 0.01
& 31.38 $\pm$ 0.20 & 28.60 $\pm$ 0.32
& 27.97 $\pm$ 0.30 & 23.10 $\pm$ 0.40
& 4.42 $\pm$ 0.41 & 3.20 $\pm$ 0.42 \\

& AttnResidual + AoS
& 11.10 $\pm$ 0.01 & 11.62 $\pm$ 0.01
& 31.58 $\pm$ 0.24 & 29.27 $\pm$ 0.32
& 28.17 $\pm$ 0.30 & 24.04 $\pm$ 0.42
& 4.46 $\pm$ 0.43 & 3.43 $\pm$ 0.44 \\

& \textbf{AttnResidual + OASIS}
& \cellcolor{LightCyan} \textbf{9.98 $\pm$ 0.01} 
& \cellcolor{LightCyan} \textbf{10.17 $\pm$ 0.01}
& \cellcolor{LightCyan} \textbf{31.76 $\pm$ 0.24} 
& \cellcolor{LightCyan} \textbf{30.19 $\pm$ 0.30}
& \cellcolor{LightCyan} \textbf{28.88 $\pm$ 0.29} 
& \cellcolor{LightCyan} \textbf{25.56 $\pm$ 0.43}
& \cellcolor{LightCyan} \textbf{4.56 $\pm$ 0.39} 
& \cellcolor{LightCyan} \textbf{3.77 $\pm$ 0.44} \\

\midrule

\multirow{5}{*}{Qwen3-0.6B}
& Original backbone
& 13.04 $\pm$ 0.01 & 13.61 $\pm$ 0.01
& 43.38 $\pm$ 0.22 & 40.06 $\pm$ 0.30
& 18.02 $\pm$ 0.29 & 15.41 $\pm$ 0.43
& 24.54 $\pm$ 0.61 & 19.59 $\pm$ 0.72 \\

& Vanilla AttnResidual
& 12.83 $\pm$ 0.01 & 44.52 $\pm$ 0.09
& 43.46 $\pm$ 0.23 & 33.21 $\pm$ 0.39
& 17.97 $\pm$ 0.33 & 12.09 $\pm$ 0.50
& 24.83 $\pm$ 0.64 & 15.41 $\pm$ 0.79 \\

& AttnResidual + OutEffHop
& 12.12 $\pm$ 0.01 & 12.94 $\pm$ 0.01
& 43.49 $\pm$ 0.20 & 38.09 $\pm$ 0.31
& 18.08 $\pm$ 0.34 & 14.56 $\pm$ 0.39
& 24.80 $\pm$ 0.59 & 18.54 $\pm$ 0.73 \\

& AttnResidual + AoS
& 11.41 $\pm$ 0.01 & 11.73 $\pm$ 0.01
& 43.73 $\pm$ 0.24 & 39.37 $\pm$ 0.31
& 18.22 $\pm$ 0.34 & 15.07 $\pm$ 0.43
& 24.97 $\pm$ 0.60 & 19.50 $\pm$ 0.71 \\

& \textbf{AttnResidual + OASIS}
& \cellcolor{LightCyan} \textbf{11.19 $\pm$ 0.01} 
& \cellcolor{LightCyan} \textbf{11.58 $\pm$ 0.01}
& \cellcolor{LightCyan} \textbf{43.77 $\pm$ 0.21} 
& \cellcolor{LightCyan} \textbf{41.00 $\pm$ 0.32}
& \cellcolor{LightCyan} \textbf{18.49 $\pm$ 0.29} 
& \cellcolor{LightCyan} \textbf{16.14 $\pm$ 0.43}
& \cellcolor{LightCyan} \textbf{25.18 $\pm$ 0.60} 
& \cellcolor{LightCyan} \textbf{21.51 $\pm$ 0.74} \\
\bottomrule
\end{tabular}}
\vspace{-0.2in}
\end{table*}

\textbf{Generalization across domains and tasks.}
We further evaluate whether the benefits of OASIS generalize beyond WikiText-2 and GSM8K. As shown in \cref{tab:generalization}, OASIS achieves the lowest perplexity on The Pile under both FP16 and W8A8 across LLaMA-3.2-1B and Qwen3-0.6B, indicating that its improvements are not specific to the Wikipedia domain. The same trend extends to downstream tasks: under W4A4 quantization, OASIS consistently outperforms Vanilla AttnResidual and the alternative stabilization methods on MMLU, TriviaQA, and MATH500, while preserving or slightly improving FP16 performance. Notably, Vanilla AttnResidual retains competitive FP16 performance but degrades sharply after quantization, whereas OASIS substantially closes this gap across all evaluated tasks. These results demonstrate that OASIS improves quantization robustness across both language-modeling domains and downstream reasoning and knowledge tasks, rather than overfitting to a particular benchmark. 
An additional ablation study of null routing and token-to-depth coupling is presented in Appendix~\ref{sec:abla_component}.

\section{Discussion and Conclusion}
\label{sec:conclusion}We propose OASIS, a null-aware routing framework that stabilizes AttnResidual under dual normalization by introducing $\Softmax_1$-based null routes at both the token and depth levels and coupling token-level null evidence back into depth routing. Across evaluated backbones and settings, OASIS reduces the maximum infinity norm by \textbf{81.75\%} and average kurtosis by \textbf{95.90\%}, lowers W8A8 perplexity by \textbf{82.00\%}, and improves W4A4 GSM8K Pass@1 by \textbf{42.11\%} on average, without modifying the training objective or overall architectural skeleton. These results show that explicit null capacity provides a simple and effective mechanism for stabilizing dual-normalized residual routing while improving robustness to quantization and long-context inference.

\clearpage

\section*{Acknowledgments}
The authors thank the anonymous reviewers and program chairs for their constructive feedback.
Haozheng Luo is partially supported by the Lambda Researcher Grant, OpenAI Research Access Program, Thinking Machine Research Grant, and Adobe Fellowship. This research was supported in part by the computational resources and staff contributions provided by Northwestern University’s Quest High Performance Computing Facility, which is jointly supported by the Office of the Provost, the Office for Research, and Northwestern University Information Technology. This work was also supported in part by computational resources provided through ACCESS allocation CIS260386.
The content is solely the responsibility of the authors and does not necessarily represent the official
views of the funding agencies.

\bibliographystyle{plainnat}
\bibliography{refs}

\begin{thebibliography}{64}
\providecommand{\natexlab}[1]{#1}
\providecommand{\url}[1]{\texttt{#1}}
\expandafter\ifx\csname urlstyle\endcsname\relax
  \providecommand{\doi}[1]{doi: #1}\else
  \providecommand{\doi}{doi: \begingroup \urlstyle{rm}\Url}\fi

\bibitem[Anand et~al.(2026)Anand, Cappellazzo, Petridis, and Pantic]{anand2026mitigating}
Anand Anand, Umberto Cappellazzo, Stavros Petridis, and Maja Pantic.
\newblock Mitigating attention sinks and massive activations in audio-visual speech recognition with llms.
\newblock In \emph{ICASSP 2026-2026 IEEE International Conference on Acoustics, Speech and Signal Processing (ICASSP)}, pages 17942--17946. IEEE, 2026.

\bibitem[Ashkboos et~al.(2024)Ashkboos, Mohtashami, Croci, Li, Cameron, Jaggi, Alistarh, Hoefler, and Hensman]{ashkboos2024quarot}
Saleh Ashkboos, Amirkeivan Mohtashami, Maximilian~L. Croci, Bo~Li, Pashmina Cameron, Martin Jaggi, Dan Alistarh, Torsten Hoefler, and James Hensman.
\newblock Quarot: Outlier-free 4-bit inference in rotated {LLM}s.
\newblock In \emph{The Thirty-eighth Annual Conference on Neural Information Processing Systems}, 2024.

\bibitem[Bai et~al.(2024)Bai, Lv, Zhang, Lyu, Tang, Huang, Du, Liu, Zeng, Hou, Dong, Tang, and Li]{bai2024longbench}
Yushi Bai, Xin Lv, Jiajie Zhang, Hongchang Lyu, Jiankai Tang, Zhidian Huang, Zhengxiao Du, Xiao Liu, Aohan Zeng, Lei Hou, Yuxiao Dong, Jie Tang, and Juanzi Li.
\newblock {L}ong{B}ench: A bilingual, multitask benchmark for long context understanding.
\newblock In Lun-Wei Ku, Andre Martins, and Vivek Srikumar, editors, \emph{Proceedings of the 62nd Annual Meeting of the Association for Computational Linguistics (Volume 1: Long Papers)}, pages 3119--3137, Bangkok, Thailand, August 2024. Association for Computational Linguistics.

\bibitem[Barbero et~al.(2025)Barbero, Arroyo, Gu, Perivolaropoulos, Veli{\v{c}}kovi{\'c}, Pascanu, and Bronstein]{barbero2025firsttoken}
Federico Barbero, Alvaro Arroyo, Xiangming Gu, Christos Perivolaropoulos, Petar Veli{\v{c}}kovi{\'c}, Razvan Pascanu, and Michael~M. Bronstein.
\newblock Why do {LLM}s attend to the first token?
\newblock In \emph{Second Conference on Language Modeling}, 2025.

\bibitem[Bondarenko et~al.(2021)Bondarenko, Nagel, and Blankevoort]{bondarenko2021understanding}
Yelysei Bondarenko, Markus Nagel, and Tijmen Blankevoort.
\newblock Understanding and overcoming the challenges of efficient transformer quantization.
\newblock In \emph{Proceedings of the 2021 Conference on Empirical Methods in Natural Language Processing}, pages 7947--7969, 2021.

\bibitem[Bondarenko et~al.(2023)Bondarenko, Nagel, and Blankevoort]{bondarenko2023quantizable}
Yelysei Bondarenko, Markus Nagel, and Tijmen Blankevoort.
\newblock Quantizable transformers: Removing outliers by helping attention heads do nothing.
\newblock \emph{Advances in Neural Information Processing Systems (NeurIPS)}, 36, 2023.

\bibitem[Chee et~al.(2023)Chee, Cai, Kuleshov, and De~Sa]{chee2023quip}
Jerry Chee, Yaohui Cai, Volodymyr Kuleshov, and Christopher~M. De~Sa.
\newblock {QuIP}: 2-bit quantization of large language models with guarantees.
\newblock In \emph{Advances in Neural Information Processing Systems}, volume~36, 2023.

\bibitem[Clark et~al.(2019)Clark, Khandelwal, Levy, and Manning]{clark2019what}
Kevin Clark, Urvashi Khandelwal, Omer Levy, and Christopher~D. Manning.
\newblock What does {BERT} look at? an analysis of {BERT}'s attention.
\newblock In \emph{Proceedings of the 2019 ACL Workshop BlackboxNLP: Analyzing and Interpreting Neural Networks for NLP}, pages 276--286, 2019.

\bibitem[Cobbe et~al.(2021)Cobbe, Kosaraju, Bavarian, Chen, Jun, Kaiser, Plappert, Tworek, Hilton, Nakano, et~al.]{cobbe2021training}
Karl Cobbe, Vineet Kosaraju, Mohammad Bavarian, Mark Chen, Heewoo Jun, Lukasz Kaiser, Matthias Plappert, Jerry Tworek, Jacob Hilton, Reiichiro Nakano, et~al.
\newblock Training verifiers to solve math word problems.
\newblock \emph{arXiv preprint arXiv:2110.14168}, 2021.

\bibitem[Dadgarnia et~al.(2026)Dadgarnia, Tabesh, Nikdan, Helcig, Kurtic, and Alistarh]{dadgarnia2026gsq}
Alireza Dadgarnia, Soroush Tabesh, Mahdi Nikdan, Michael Helcig, Eldar Kurtic, and Dan Alistarh.
\newblock Gsq: Highly-accurate low-precision scalar quantization for llms via gumbel-softmax sampling.
\newblock \emph{arXiv preprint arXiv:2604.18556}, 2026.

\bibitem[Dehghani et~al.(2019)Dehghani, Gouws, Vinyals, Uszkoreit, and Kaiser]{dehghani2019universal}
Mostafa Dehghani, Stephan Gouws, Oriol Vinyals, Jakob Uszkoreit, and {\L}ukasz Kaiser.
\newblock Universal transformers.
\newblock In \emph{International Conference on Learning Representations}, 2019.

\bibitem[Dettmers et~al.(2022)Dettmers, Lewis, Belkada, and Zettlemoyer]{llm.int.8}
Tim Dettmers, Mike Lewis, Younes Belkada, and Luke Zettlemoyer.
\newblock Gpt3.int8(): 8-bit matrix multiplication for transformers at scale.
\newblock In S.~Koyejo, S.~Mohamed, A.~Agarwal, D.~Belgrave, K.~Cho, and A.~Oh, editors, \emph{Advances in Neural Information Processing Systems}, volume~35, pages 30318--30332. Curran Associates, Inc., 2022.

\bibitem[Dettmers et~al.(2023)Dettmers, Pagnoni, Holtzman, and Zettlemoyer]{dettmers2024qlora}
Tim Dettmers, Artidoro Pagnoni, Ari Holtzman, and Luke Zettlemoyer.
\newblock Qlora: Efficient finetuning of quantized llms.
\newblock In \emph{The Thirty-seventh Conference on Neural Information Processing Systems (NeurIPS)}, 2023.

\bibitem[Dettmers et~al.(2024)Dettmers, Svirschevski, Egiazarian, Kuznedelev, Frantar, Ashkboos, Borzunov, Hoefler, and Alistarh]{dettmers2024spqr}
Tim Dettmers, Ruslan~A. Svirschevski, Vage Egiazarian, Denis Kuznedelev, Elias Frantar, Saleh Ashkboos, Alexander Borzunov, Torsten Hoefler, and Dan Alistarh.
\newblock {SpQR}: A sparse-quantized representation for near-lossless {LLM} weight compression.
\newblock In \emph{International Conference on Learning Representations}, 2024.

\bibitem[Fan et~al.(2020)Fan, Grave, and Joulin]{fan2020layerdrop}
Angela Fan, Edouard Grave, and Armand Joulin.
\newblock Reducing transformer depth on demand with structured dropout.
\newblock In \emph{International Conference on Learning Representations}, 2020.

\bibitem[Frantar et~al.(2023)Frantar, Ashkboos, Hoefler, and Alistarh]{frantar2023gptq}
Elias Frantar, Saleh Ashkboos, Torsten Hoefler, and Dan Alistarh.
\newblock {GPTQ}: Accurate post-training quantization for generative pre-trained transformers.
\newblock In \emph{International Conference on Learning Representations}, 2023.

\bibitem[Gholami et~al.(2022)Gholami, Kim, Dong, Yao, Mahoney, and Keutzer]{gholami2022survey}
Amir Gholami, Sehoon Kim, Zhen Dong, Zhewei Yao, Michael~W Mahoney, and Kurt Keutzer.
\newblock A survey of quantization methods for efficient neural network inference.
\newblock In \emph{Low-power computer vision}, pages 291--326. Chapman and Hall/CRC, 2022.

\bibitem[Gomez(2026)]{gomez2026attnres}
Kye Gomez.
\newblock attn\_res: Implementation of attention residuals.
\newblock \url{https://github.com/kyegomez/attn_res}, 2026.
\newblock GitHub repository. Unofficial PyTorch implementation of Attention Residuals. Accessed: 2026-04-08.

\bibitem[Grattafiori et~al.(2024)Grattafiori, Dubey, Jauhri, Pandey, Kadian, Al-Dahle, Letman, Mathur, Schelten, Vaughan, et~al.]{grattafiori2024LLaMA}
Aaron Grattafiori, Abhimanyu Dubey, Abhinav Jauhri, Abhinav Pandey, Abhishek Kadian, Ahmad Al-Dahle, Aiesha Letman, Akhil Mathur, Alan Schelten, Alex Vaughan, et~al.
\newblock The llama 3 herd of models.
\newblock \emph{arXiv preprint arXiv:2407.21783}, 2024.

\bibitem[Gu et~al.(2025)Gu, Pang, Du, Liu, Zhang, Du, Wang, and Lin]{gu2025attentionsink}
Xiangming Gu, Tianyu Pang, Chao Du, Qian Liu, Fengzhuo Zhang, Cunxiao Du, Ye~Wang, and Min Lin.
\newblock When attention sink emerges in language models: An empirical view.
\newblock In \emph{International Conference on Learning Representations}, 2025.

\bibitem[Guo et~al.(2020)Guo, Dai, Vrande{\v{c}}i{\'c}, and Al-Rfou]{guo2020wiki}
Mandy Guo, Zihang Dai, Denny Vrande{\v{c}}i{\'c}, and Rami Al-Rfou.
\newblock Wiki-40b: Multilingual language model dataset.
\newblock In \emph{Proceedings of the Twelfth Language Resources and Evaluation Conference}, pages 2440--2452, 2020.

\bibitem[He et~al.(2024)He, Luo, and Wang]{he2024st}
Haoyu He, Haozheng Luo, and Qi~R Wang.
\newblock St-moe-bert: A spatial-temporal mixture-of-experts framework for long-term cross-city mobility prediction.
\newblock In \emph{Proceedings of the 2nd ACM SIGSPATIAL International Workshop on Human Mobility Prediction Challenge}, pages 10--15, 2024.

\bibitem[He et~al.(2025)He, Luo, Chen, and Wang]{he2025rhythm}
Haoyu He, Haozheng Luo, Yan Chen, and Qi~Wang.
\newblock {RHYTHM}: Reasoning with hierarchical temporal tokenization for human mobility.
\newblock In \emph{The Thirty-ninth Annual Conference on Neural Information Processing Systems}, 2025.

\bibitem[Hendrycks et~al.(2021)Hendrycks, Burns, Kadavath, Arora, Basart, Tang, Song, and Steinhardt]{hendrycks2021measuring}
Dan Hendrycks, Collin Burns, Saurav Kadavath, Akul Arora, Steven Basart, Eric Tang, Dawn Song, and Jacob Steinhardt.
\newblock Measuring mathematical problem solving with the {MATH} dataset.
\newblock In \emph{Thirty-fifth Conference on Neural Information Processing Systems Datasets and Benchmarks Track (Round 2)}, 2021.

\bibitem[Hosseini et~al.(2026)Hosseini, Ardakani, and Gross]{hosseini2026innerq}
Sayed Mohammadreza~Tayaranian Hosseini, Amir Ardakani, and Warren~J. Gross.
\newblock Innerq: Hardware-aware tuning-free quantization of kv cache for large language models.
\newblock \emph{arXiv preprint arXiv:2602.23200}, 2026.

\bibitem[Hsieh et~al.(2024)Hsieh, Sun, Kriman, Acharya, Rekesh, Jia, and Ginsburg]{hsieh2024ruler}
Cheng-Ping Hsieh, Simeng Sun, Samuel Kriman, Shantanu Acharya, Dima Rekesh, Fei Jia, and Boris Ginsburg.
\newblock {RULER}: What{\textquoteright}s the real context size of your long-context language models?
\newblock In \emph{First Conference on Language Modeling}, 2024.

\bibitem[Hu et~al.(2024)Hu, Chang, Luo, Chen, Li, Wang, and Liu]{hu2024outlier}
Jerry Yao-Chieh Hu, Pei-Hsuan Chang, Robin Luo, Hong-Yu Chen, Weijian Li, Wei-Po Wang, and Han Liu.
\newblock Outlier-efficient hopfield layers for large transformer-based models.
\newblock In \emph{The Forty-first International Conference on Machine Learning (ICML)}, 2024.

\bibitem[Joshi et~al.(2017)Joshi, Choi, Weld, and Zettlemoyer]{joshi2017triviaqa}
Mandar Joshi, Eunsol Choi, Daniel Weld, and Luke Zettlemoyer.
\newblock {T}rivia{QA}: A large scale distantly supervised challenge dataset for reading comprehension.
\newblock In Regina Barzilay and Min-Yen Kan, editors, \emph{Proceedings of the 55th Annual Meeting of the Association for Computational Linguistics (Volume 1: Long Papers)}, pages 1601--1611, Vancouver, Canada, July 2017. Association for Computational Linguistics.

\bibitem[Kang et~al.(2025)Kang, Kim, Kim, and Hwang]{kang2025see}
Seil Kang, Jinyeong Kim, Junhyeok Kim, and Seong~Jae Hwang.
\newblock See what you are told: Visual attention sink in large multimodal models.
\newblock In \emph{The Thirteenth International Conference on Learning Representations}, 2025.

\bibitem[Kaul et~al.(2025)Kaul, Ma, Elezi, and Deng]{kaul2025attentionactivation}
Prannay Kaul, Chengcheng Ma, Ismail Elezi, and Jiankang Deng.
\newblock From attention to activation: Unraveling the enigmas of large language models.
\newblock In \emph{The Thirteenth International Conference on Learning Representations}, 2025.

\bibitem[Kobayashi et~al.(2020)Kobayashi, Kuribayashi, Yokoi, and Inui]{kobayashi2020attention}
Goro Kobayashi, Tatsuki Kuribayashi, Sho Yokoi, and Kentaro Inui.
\newblock Attention is not only a weight: Analyzing transformers with vector norms.
\newblock In \emph{Proceedings of the 2020 Conference on Empirical Methods in Natural Language Processing}, pages 7057--7075, 2020.

\bibitem[Kovaleva et~al.(2019)Kovaleva, Romanov, Rogers, and Rumshisky]{kovaleva2019dark}
Olga Kovaleva, Alexey Romanov, Anna Rogers, and Anna Rumshisky.
\newblock Revealing the dark secrets of {BERT}.
\newblock In \emph{Proceedings of the 2019 Conference on Empirical Methods in Natural Language Processing and the 9th International Joint Conference on Natural Language Processing}, pages 4365--4374, 2019.

\bibitem[Lee et~al.(2024)Lee, Jin, Kim, Kim, and Park]{lee2024owq}
Changhun Lee, Jungyu Jin, Taesu Kim, Hyungjun Kim, and Eunhyeok Park.
\newblock {OWQ}: Outlier-aware weight quantization for efficient fine-tuning and inference of large language models.
\newblock In \emph{Proceedings of the AAAI Conference on Artificial Intelligence}, volume~38, pages 13355--13364, 2024.

\bibitem[Li et~al.(2024)Li, Yu, Zhang, Liang, He, Chen, and Zhao]{li2024loftq}
Yixiao Li, Yuxin Yu, Qingru Zhang, Chao Liang, Pengcheng He, Weizhu Chen, and Tuo Zhao.
\newblock {LoftQ}: {LoRA}-fine-tuning-aware quantization for large language models.
\newblock In \emph{International Conference on Learning Representations}, 2024.

\bibitem[Lin et~al.(2024)Lin, Tang, Tang, Yang, Chen, Wang, Xiao, Dang, Gan, and Han]{lin2024awq}
Ji~Lin, Jiaming Tang, Haotian Tang, Shang Yang, Wei-Ming Chen, Wei-Chen Wang, Guangxuan Xiao, Xingyu Dang, Chuang Gan, and Song Han.
\newblock {AWQ}: Activation-aware weight quantization for {LLM} compression and acceleration.
\newblock \emph{Proceedings of Machine Learning and Systems}, 6:\penalty0 87--100, 2024.

\bibitem[Liu et~al.(2024)Liu, Li, Li, Li, Zhang, Shen, and Lee]{liu2024llava}
Haotian Liu, Chunyuan Li, Yuheng Li, Bo~Li, Yuanhan Zhang, Sheng Shen, and Yong~Jae Lee.
\newblock Llava-next: Improved reasoning, ocr, and world knowledge, 2024.

\bibitem[Liu et~al.(2020)Liu, Zhou, Zhao, Wang, Deng, and Ju]{liu2020fastbert}
Weijie Liu, Peng Zhou, Zhe Zhao, Zhiruo Wang, Haotang Deng, and Qi~Ju.
\newblock {FastBERT}: a self-distilling {BERT} with adaptive inference time.
\newblock In \emph{Proceedings of the 58th Annual Meeting of the Association for Computational Linguistics}, pages 6035--6044, 2020.

\bibitem[Liu et~al.(2025)Liu, Zhao, Fedorov, Soran, Choudhary, Krishnamoorthi, Chandra, Tian, and Blankevoort]{liu2024spinquant}
Zechun Liu, Changsheng Zhao, Igor Fedorov, Bilge Soran, Dhruv Choudhary, Raghuraman Krishnamoorthi, Vikas Chandra, Yuandong Tian, and Tijmen Blankevoort.
\newblock Spinquant: {LLM} quantization with learned rotations.
\newblock In \emph{The Thirteenth International Conference on Learning Representations}, 2025.

\bibitem[Luo et~al.(2025)Luo, Qiu, Su, Zhou, Mehta, Ye, Hu, and Liu]{luo2025fast}
Haozheng Luo, Chenghao Qiu, Maojiang Su, Zhihan Zhou, Zoe Mehta, Guo Ye, Jerry Yao-Chieh Hu, and Han Liu.
\newblock Fast and low-cost genomic foundation models via outlier removal.
\newblock In \emph{Forty-second International Conference on Machine Learning}, 2025.

\bibitem[Luo et~al.(2026{\natexlab{a}})Luo, Jiang, Hasan, Chen, and Sarkar]{luo2026frost}
Haozheng Luo, Zhuolin Jiang, Md~Zahid Hasan, Yan Chen, and Soumalya Sarkar.
\newblock {FROST}: Filtering reasoning outliers with attention for efficient reasoning.
\newblock In \emph{The Fourteenth International Conference on Learning Representations}, 2026{\natexlab{a}}.

\bibitem[Luo et~al.(2026{\natexlab{b}})Luo, Wang, Yu, Wang, and Chen]{luo2026contrastive}
Haozheng Luo, Yimin Wang, Jiahao Yu, Binghui Wang, and Yan Chen.
\newblock Contrastive reasoning alignment: Reinforcement learning from hidden representations.
\newblock In \emph{Forty-third International Conference on Machine Learning}, 2026{\natexlab{b}}.

\bibitem[Luo et~al.(2026{\natexlab{c}})Luo, Fan, Wang, He, Rahman, Abolmaesumi, and Sigal]{luo2026to}
Jiayun Luo, Wan-Cyuan Fan, Lyuyang Wang, Xiangteng He, Tanzila Rahman, Purang Abolmaesumi, and Leonid Sigal.
\newblock To sink or not to sink: Visual information pathways in large vision-language models.
\newblock In \emph{The Fourteenth International Conference on Learning Representations}, 2026{\natexlab{c}}.

\bibitem[Merity et~al.(2017)Merity, Xiong, Bradbury, and Socher]{merity2016pointer}
Stephen Merity, Caiming Xiong, James Bradbury, and Richard Socher.
\newblock Pointer sentinel mixture models.
\newblock In \emph{The Fifth Conference on International Conference on Learning Representations (ICLR)}, 2017.

\bibitem[Miller(2023)]{miller2021}
Evan Miller.
\newblock Blog post: Attention is off by one, 2023.
\newblock URL \url{https://www.evanmiller.org/attention-is-off-by-one.html}.
\newblock Accessed: July 4, 2024.

\bibitem[Pan et~al.(2025)Pan, Luo, Li, and Liu]{pan2025chainofaction}
Zhenyu Pan, Haozheng Luo, Manling Li, and Han Liu.
\newblock Chain-of-action: Faithful and multimodal question answering through large language models.
\newblock In \emph{The Thirteenth International Conference on Learning Representations}, 2025.

\bibitem[Ran-Milo(2026)]{ranmilo2026sinks}
Yuval Ran-Milo.
\newblock Attention sinks are provably necessary in softmax transformers: Evidence from trigger-conditional tasks.
\newblock In Maria Liakata, Viviane~P. Moreira, Jiajun Zhang, and David Jurgens, editors, \emph{Proceedings of the 64th Annual Meeting of the {A}ssociation for {C}omputational {L}inguistics (Volume 2: Short Papers)}, pages 70--89, San Diego, California, United States, July 2026. Association for Computational Linguistics.
\newblock ISBN 979-8-89176-391-3.

\bibitem[Raposo et~al.(2024)Raposo, Ritter, Richards, Lillicrap, Humphreys, and Santoro]{raposo2024mixture}
David Raposo, Sam Ritter, Blake Richards, Timothy Lillicrap, Peter~Conway Humphreys, and Adam Santoro.
\newblock Mixture-of-depths: Dynamically allocating compute in transformer-based language models.
\newblock \emph{arXiv preprint arXiv:2404.02258}, 2024.

\bibitem[Shao et~al.(2024)Shao, Chen, Zhang, Xu, Zhao, Li, Zhang, Gao, Qiao, and Luo]{shao2024omniquant}
Wenqi Shao, Mengzhao Chen, Zhaoyang Zhang, Peng Xu, Lirui Zhao, Zhiqian Li, Kaipeng Zhang, Peng Gao, Yu~Qiao, and Ping Luo.
\newblock Omniquant: Omnidirectionally calibrated quantization for large language models.
\newblock In \emph{The Twelfth International Conference on Learning Representations}, 2024.

\bibitem[Shkolnik et~al.(2020)Shkolnik, Chmiel, Banner, Shomron, Nahshan, Bronstein, and Weiser]{shkolnik2020robust}
Moran Shkolnik, Brian Chmiel, Ron Banner, Gil Shomron, Yury Nahshan, Alex Bronstein, and Uri Weiser.
\newblock Robust quantization: One model to rule them all.
\newblock In H.~Larochelle, M.~Ranzato, R.~Hadsell, M.F. Balcan, and H.~Lin, editors, \emph{Advances in Neural Information Processing Systems}, volume~33, pages 5308--5317. Curran Associates, Inc., 2020.

\bibitem[Son et~al.(2024)Son, Park, Han, Kim, and Lee]{son2024prefixing}
Seungwoo Son, Wonpyo Park, Woohyun Han, Kyuyeun Kim, and Jaeho Lee.
\newblock Prefixing attention sinks can mitigate activation outliers for large language model quantization.
\newblock In \emph{Proceedings of the 2024 Conference on Empirical Methods in Natural Language Processing}, 2024.

\bibitem[Su and Yuan(2025)]{su2025kvsink}
Zunhai Su and Kehong Yuan.
\newblock {KVS}ink: Understanding and enhancing the preservation of attention sinks in {KV} cache quantization for {LLM}s.
\newblock In \emph{Second Conference on Language Modeling}, 2025.

\bibitem[Sun et~al.(2025)Sun, Liu, Bai, Bao, Zhao, Li, JiaxinHu, Yu, Hou, Yuan, Jiang, Liu, and Yao]{sun2024flatquant}
Yuxuan Sun, Ruikang Liu, Haoli Bai, Han Bao, Kang Zhao, Yuening Li, JiaxinHu, Xianzhi Yu, Lu~Hou, Chun Yuan, Xin Jiang, Wulong Liu, and Jun Yao.
\newblock Flatquant: Flatness matters for {LLM} quantization.
\newblock In \emph{Forty-second International Conference on Machine Learning}, 2025.

\bibitem[Team et~al.(2026)Team, Chen, Zhang, Su, Xu, Pan, Wang, Wang, Chen, Yin, et~al.]{team2026attention}
Kimi Team, Guangyu Chen, Yu~Zhang, Jianlin Su, Weixin Xu, Siyuan Pan, Yaoyu Wang, Yucheng Wang, Guanduo Chen, Bohong Yin, et~al.
\newblock Attention residuals.
\newblock \emph{arXiv preprint arXiv:2603.15031}, 2026.

\bibitem[Vaswani et~al.(2017)Vaswani, Shazeer, Parmar, Uszkoreit, Jones, Gomez, Kaiser, and Polosukhin]{vaswani2017attention}
Ashish Vaswani, Noam Shazeer, Niki Parmar, Jakob Uszkoreit, Llion Jones, Aidan~N. Gomez, \L{}ukasz Kaiser, and Illia Polosukhin.
\newblock Attention is all you need.
\newblock In \emph{The Thirty-first Conference in Neural Information Processing Systems (NeurIPS)}, 2017.

\bibitem[Wei et~al.(2022)Wei, Zhang, Zhang, Gong, Zhang, Zhang, Yu, and Liu]{wei2022outlier}
Xiuying Wei, Yunchen Zhang, Xiangguo Zhang, Ruihao Gong, Shanghang Zhang, Qi~Zhang, Fengwei Yu, and Xianglong Liu.
\newblock Outlier suppression: Pushing the limit of low-bit transformer language models.
\newblock \emph{Advances in Neural Information Processing Systems}, 35:\penalty0 17402--17414, 2022.

\bibitem[Wong et~al.(2026)Wong, Zhang, Mahon, Luk, Isopoussu, and Zhao]{wong2025existence}
Jeffrey T.~H. Wong, Cheng Zhang, Louis Mahon, Wayne Luk, Anton Isopoussu, and Yiren Zhao.
\newblock On the existence and behavior of secondary attention sinks.
\newblock In \emph{ICLR 2026 Workshop on Unifying Concept Representation Learning}, 2026.

\bibitem[Xiao et~al.(2023)Xiao, Lin, Seznec, Wu, Demouth, and Han]{xiao2023smoothquant}
Guangxuan Xiao, Ji~Lin, Mickael Seznec, Hao Wu, Julien Demouth, and Song Han.
\newblock {SmoothQuant}: Accurate and efficient post-training quantization for large language models.
\newblock In \emph{Proceedings of the 40th International Conference on Machine Learning}, pages 38087--38099, 2023.

\bibitem[Xiao et~al.(2024)Xiao, Tian, Chen, Han, and Lewis]{xiao2023efficient}
Guangxuan Xiao, Yuandong Tian, Beidi Chen, Song Han, and Mike Lewis.
\newblock Efficient streaming language models with attention sinks.
\newblock In \emph{The Twelfth International Conference on Learning Representations}, 2024.

\bibitem[Xiao et~al.(2025)Xiao, Tang, Zuo, Guo, Yang, Tang, Fu, and Han]{xiao2025duoattention}
Guangxuan Xiao, Jiaming Tang, Jingwei Zuo, Junxian Guo, Shang Yang, Haotian Tang, Yao Fu, and Song Han.
\newblock Duoattention: Efficient long-context {LLM} inference with retrieval and streaming heads.
\newblock In \emph{International Conference on Learning Representations}, 2025.

\bibitem[Yang et~al.(2025)Yang, Li, Yang, Zhang, Hui, Zheng, Yu, Gao, Huang, Lv, et~al.]{yang2025qwen3}
An~Yang, Anfeng Li, Baosong Yang, Beichen Zhang, Binyuan Hui, Bo~Zheng, Bowen Yu, Chang Gao, Chengen Huang, Chenxu Lv, et~al.
\newblock Qwen3 technical report.
\newblock \emph{arXiv preprint arXiv:2505.09388}, 2025.

\bibitem[Zhang et~al.(2023)Zhang, Sheng, Zhou, Chen, Zheng, Cai, Song, Tian, R{\'e}, Barrett, Wang, and Chen]{zhang2023h2o}
Zhenyu Zhang, Ying Sheng, Tianyi Zhou, Tianlong Chen, Lianmin Zheng, Ruisi Cai, Zhao Song, Yuandong Tian, Christopher R{\'e}, Clark Barrett, Zhangyang Wang, and Beidi Chen.
\newblock {H$_2$O}: Heavy-hitter oracle for efficient generative inference of large language models.
\newblock In \emph{Advances in Neural Information Processing Systems}, volume~36, 2023.

\bibitem[Zhou et~al.(2020)Zhou, Xu, Ge, McAuley, Xu, and Wei]{zhou2020bertlosespatience}
Wangchunshu Zhou, Canwen Xu, Tao Ge, Julian McAuley, Ke~Xu, and Furu Wei.
\newblock {BERT} loses patience: Fast and robust inference with early exit.
\newblock In \emph{Advances in Neural Information Processing Systems}, volume~33, 2020.

\bibitem[Zhu et~al.(2015)Zhu, Kiros, Zemel, Salakhutdinov, Urtasun, Torralba, and Fidler]{Zhu_2015_ICCV}
Yukun Zhu, Ryan Kiros, Rich Zemel, Ruslan Salakhutdinov, Raquel Urtasun, Antonio Torralba, and Sanja Fidler.
\newblock Aligning books and movies: Towards story-like visual explanations by watching movies and reading books.
\newblock In \emph{The IEEE International Conference on Computer Vision (ICCV)}, December 2015.

\bibitem[Zuhri et~al.(2026)Zuhri, Fuadi, and Aji]{zuhri2025softpick}
Zayd Muhammad~Kawakibi Zuhri, Erland~Hilman Fuadi, and Alham~Fikri Aji.
\newblock Softpick: No attention sink, no massive activations with rectified softmax.
\newblock In Maria Liakata, Viviane~P. Moreira, Jiajun Zhang, and David Jurgens, editors, \emph{Findings of the {A}ssociation for {C}omputational {L}inguistics: {ACL} 2026}, pages 18179--18213, San Diego, California, United States, July 2026. Association for Computational Linguistics.
\newblock ISBN 979-8-89176-395-1.

\end{thebibliography}

\newpage  %
\normalsize
\setlist[itemize]{leftmargin=1em}
\setlist[enumerate]{leftmargin=1.4em}

\setlength{\abovedisplayskip}{10pt}
\setlength{\abovedisplayshortskip}{10pt}
\setlength{\belowdisplayskip}{10pt}
\setlength{\belowdisplayshortskip}{10pt}

\part*{Appendix}
\appendix	

{
\setlength{\parskip}{0.3em}
\startcontents[sections]
\printcontents[sections]{ }{1}{}
}

{
\setlength{\parskip}{-0em}
\startcontents[sections]
\printcontents[sections]{ }{1}{}
}

\section{Impact Statement}
\label{ap:impact}
We believe that this methodology can strengthen the core robustness of foundation models, including large language models, by improving quantization and mitigating attention sinks. However, such modifications may also amplify biases present in the training data, potentially leading to unfair or discriminatory outcomes for underrepresented groups.

\section{Additional Related Work}
\label{ap:related_works}
\textbf{Attention Sinks.}
Attention sinks~\citep{xiao2023efficient} reflect a normalization-induced bias in attention allocation. They emerge during training, exhibit layer-wise hierarchy, for example, primary vs. secondary~\citep{wong2025existence,gu2025attentionsink}, and are closely tied to first-token dominance, which stabilizes representations by limiting over-mixing~\citep{barbero2025firsttoken}. Both empirical and theoretical work links sinks to the feasible set under softmax normalization, showing they can be functionally necessary~\citep{ranmilo2026sinks}. Accordingly, StreamingLLM~\citep{xiao2023efficient} and follow-ups~\citep{wong2025existence,gu2025attentionsink} leverage sink structure for long-context modeling. Similar sink-like concentration patterns are also observed in multimodal transformers, including vision~\citep{luo2026to,kang2025see} and audio~\citep{anand2026mitigating} models, revealing the fact that sink formation is a modality-agnostic consequence of attention normalization.
At the system level, sinks motivate efficient designs that either preserve or suppress them. H$_2$O~\citep{zhang2023h2o} exploits high-attention heavy hitters for KV-cache eviction, while DuoAttention~\citep{xiao2025duoattention} and KVSink~\citep{su2025kvsink} more explicitly preserve sink structure for efficient inference. Prefixing Attention Sinks~\citep{son2024prefixing} and Softpick~\citep{zuhri2025softpick} aim to mitigate them via architectural or normalization changes. In contrast, building on AttnResidual's softmax-based depth aggregation~\citep{team2026attention}, we argue that coupled token- and depth-level normalization constraints force near no-op updates onto real tokens and branches, amplifying sink accumulation. We therefore treat sinks as a robustness issue and mitigate them via explicit null destinations and token-to-depth coupling.

\textbf{Quantization.}
Quantization methods for transformer-based foundation models are commonly divided into weight-only quantization and weight-activation quantization \cite{gholami2022survey}. In the weight-only setting, GPTQ~\citep{frantar2023gptq} performs post-training quantization through block-wise second-order reconstruction. SpQR~\citep{dettmers2024spqr} combines sparsity with quantization to improve the compression--accuracy trade-off. OWQ~\citep{lee2024owq} protects outlier-related weights through outlier-aware weight quantization. AWQ~\citep{lin2024awq} emphasizes activation-aware scaling and avoids part of the hardware inefficiency of mixed-precision storage. QUIP~\citep{chee2023quip} studies very low-bit quantization with theoretical guarantees. QLoRA~\citep{dettmers2024qlora} restores capability through low-rank adaptation on quantized backbones, while LoftQ~\citep{li2024loftq} makes quantization explicitly aware of later LoRA fine-tuning.

For weight–activation quantization, prior work jointly compresses weights and activations, with several approaches proposed; for instance, LLM.int8()~\citep{llm.int.8} uses mixed-precision decomposition to address the problem. Outlier Suppression~\citep{wei2022outlier}, on the other hand, relies on migration and clipping to reduce the imbalance in activation ranges. SmoothQuant~\citep{xiao2023smoothquant} takes a different approach by migrating some of the activation difficulty into the weights through an equivalent transformation. More recently, OmniQuant~\citep{shao2024omniquant} has made significant improvements to low-bit post-training quantization through omnidirectional calibration. QuaRot~\citep{ashkboos2024quarot} is another technique that removes hidden-state outliers through rotation, enabling end-to-end low-bit quantization of weights, activations, and KV cache. Building on this, SpinQuant~\citep{liu2024spinquant} replaces fixed rotations with learned rotations to further improve quantized accuracy. Meanwhile, FlatQuant~\citep{sun2024flatquant} argues that flatness remains crucial even after prior transformations and introduces learnable affine transformations to flatten weights and activations. One approach that is particularly relevant to our setting is Prefixing Attention Sinks~\citep{son2024prefixing}. Instead of just calibrating activations, it shows that inserting sink-like prefix states can directly mitigate activation outliers and improve activation quantization. Other recent work on KV-cache has also highlighted the connection between sink preservation and quantization robustness. For example, KVSink~\citep{su2025kvsink} explores how attention sinks interact with KV-cache quantization and preserves predicted sink tokens during compression. InnerQ~\citep{hosseini2026innerq} takes this a step further by incorporating high-precision windows for both recent tokens and sink tokens in a hardware-aware KV-cache quantization scheme. What sets our work apart is that we target the structural origin of outlier-prone activations in AttnResidual~\cite{team2026attention}—namely, dual-normalized routing—rather than relying on post hoc calibration or correction.

\textbf{Attention Residuals.}
A separate literature studies how computation should be allocated across depth. Universal Transformer~\citep{dehghani2019universal} performs iterative refinement with shared parameters across steps. LayerDrop~\citep{fan2020layerdrop} enables structured layer removal and reduced-depth inference. FastBERT~\citep{liu2020fastbert} uses self-distillation for adaptive inference depth. PABEE~\citep{zhou2020bertlosespatience} introduces patience-based early exit for faster inference. More recently, Mixture-of-Depths~\citep{raposo2024mixture} dynamically allocates computation across tokens and layers under an explicit compute budget. These methods all relax the use of fixed-depth computation, but they do so through halting, skipping, or conditional execution.

When it comes to computation, AttnResidual~\citep{team2026attention} takes a distinct approach. Rather than deciding when to stop computing, it replaces the traditional method of accumulating residuals with a learned approach that aggregates the outputs of previous layers. This design allows for more flexibility in routing, but it also adds an extra layer of normalization. Our work builds on this foundation. Specifically, we explore how the dual-normalization design of AttnResidual can amplify outliers, affect sink accumulation, and create pressure on no-op operations. To address these challenges, we propose a null-aware routing mechanism that makes both token-level and depth-level aggregation more robust.

\section{Proofs of Main Text}
\label{ap:proof}
\subsection{\cref{lemma:1}}
\label{proof:LLaMA1}

{\setstretch{1}
\setlength{\abovedisplayskip}{4pt}
\setlength{\belowdisplayskip}{4pt}
\setlength{\abovedisplayshortskip}{2pt}
\setlength{\belowdisplayshortskip}{2pt}
\textit{Proof of \cref{lemma:1}.} 
Write
$$
A_t^{(\ell)}=p_{t,o}^{(\ell)}v_o^{(\ell)}+\sum_{j\neq o}p_{t,j}^{(\ell)}v_j^{(\ell)}.
$$
By the alignment assumption, for the unit vector $u$ we have
$$
\Big\langle \sum_{j\neq o}p_{t,j}^{(\ell)}v_j^{(\ell)},u\Big\rangle
=
\sum_{j\neq o}p_{t,j}^{(\ell)}\langle v_j^{(\ell)},u\rangle
\ge
c_1\sum_{j\neq o}p_{t,j}^{(\ell)}.
$$
Hence
$$
\Big\|\sum_{j\neq o}p_{t,j}^{(\ell)}v_j^{(\ell)}\Big\|
\ge
c_1\sum_{j\neq o}p_{t,j}^{(\ell)}.
$$
Using
$$
\|A_t^{(\ell)}\|
=
\Big\|\sum_{j\neq o}p_{t,j}^{(\ell)}v_j^{(\ell)}+p_{t,o}^{(\ell)}v_o^{(\ell)}\Big\|
\ge
\Big\|\sum_{j\neq o}p_{t,j}^{(\ell)}v_j^{(\ell)}\Big\|-p_{t,o}^{(\ell)}\|v_o^{(\ell)}\|,
$$
we obtain
$$
\|A_t^{(\ell)}\|\ge c_1\sum_{j\neq o}p_{t,j}^{(\ell)}-c_0p_{t,o}^{(\ell)}
=c_1-(c_0+c_1)p_{t,o}^{(\ell)}.
$$
Based on \cref{assum:1}, if $\|A_t^{(\ell)}\|\le\delta$, then
$$
\delta\ge c_1-(c_0+c_1)p_{t,o}^{(\ell)}
\iff
p_{t,o}^{(\ell)}\ge\frac{c_1-\delta}{c_0+c_1}.
$$
}

This establishes a necessary lower bound on the attention mass assigned
to the low-norm token under the stated separation assumption. The bound approaches one as both $c_0/c_1$ and $\delta/c_1$ approach zero, characterizing the attention-concentration mechanism associated with near-no-op behavior studied in OutEffHop~\citep{hu2024outlier}.

\subsection{\cref{proposition:attnresidual_pathology}}
\label{proof:therm1}
\textit{Proof of \cref{proposition:attnresidual_pathology}.} 
By \cref{assum:3}, the AttnResidual leakage, concentration, and entropy-collapse terms satisfy
$$
L_t^{\mathrm{AR}}\ge L_t^{\mathrm{V}},
\qquad
C_t^{\mathrm{AR}}\ge C_t^{\mathrm{V}},
\qquad
E_t^{\mathrm{AR}}\ge E_t^{\mathrm{V}}.
$$
Since the pathology scores are defined by
$$
\mathcal{S}_t^{\mathrm{AttnResidual}}=L_t^{\mathrm{AR}}+\lambda_1 C_t^{\mathrm{AR}}+\lambda_2 E_t^{\mathrm{AR}},
$$
and
$$
\mathcal{S}_t^{\mathrm{Vanilla}}=L_t^{\mathrm{V}}+\lambda_1 C_t^{\mathrm{V}}+\lambda_2 E_t^{\mathrm{V}},
$$
with nonnegative coefficients $\lambda_1,\lambda_2$, it follows immediately that
$$
\mathcal{S}_t^{\mathrm{AttnResidual}}\ge \mathcal{S}_t^{\mathrm{Vanilla}}.
$$

This proves the theorem. Note that \cref{proposition:attnresidual_pathology} is conditional: it aggregates the three routing pathologies when their component-wise dominance conditions hold, but does not assert that these conditions hold for every token or backbone. Table~\ref{tab:assumption_check} evaluates these conditions empirically. We find that entropy-collapse dominance is consistent across the evaluated models, whereas leakage and peak-concentration dominance are model dependent.

\subsection{\cref{thm:therm2}}
\label{proof:therm2}
\textit{Proof of \cref{thm:therm2}.}
We follow the residual-view analysis in OutEffHop~\cite{hu2024outlier}: for tokens in $\mathcal{T}_{\mathrm{noop}}$, stable optimization favors a small deviation from a reference state. Let $s_{0,t}=h_t^{(1)}$ and $s_{i,t}=f_i(h_t^{(i)})$ for $i\geq1$ denote the states supplied by the source branches, and define their deviations from a reference state $h_t^{\mathrm{ref}}$ as
\[
u_{i,t}:=s_{i,t}-h_t^{\mathrm{ref}}.
\]
Since the depth-routing weights satisfy $\sum_i\alpha_{i\to\ell,t}=1$, the depth-composed deviation is
\[
\Delta h_t^{(\ell)}:=h_t^{(\ell)}-h_t^{\mathrm{ref}}=\sum_{i=0}^{\ell-1}\alpha_{i\to\ell,t}u_{i,t}.
\]

For no-op tokens, the near-reference condition requires
$$
\|\Delta h_t^{(\ell)}\|\le \delta_t,
\qquad t\in\mathcal{T}_{\mathrm{noop}},
$$
with small $\delta_t$. Because depth coefficients satisfy simplex constraints
($\alpha_{i\to\ell,t}\ge 0$, $\sum_i\alpha_{i\to\ell,t}=1$), this becomes a constrained routing problem in which a small composed deviation may require substantial routing mass on source branches with small deviations from the reference state, i.e., no-op routes.

Let
$$
i_t^*\in\arg\min_{i\in\{0,\dots,\ell-1\}} \|u_{i,t}\|.
$$
By assumption, there exists a unit vector $w_t^{(\ell)}\in\mathbb{R}^d$ such that
$$
\langle u_{i,t},w_t^{(\ell)}\rangle\ge b_1,
\qquad i\neq i_t^*,
$$
while the minimal-update branch satisfies $\|u_{i_t^*,t}\|\le b_0$.
Projecting the residual update onto $w_t^{(\ell)}$ gives
$$
\left\langle \Delta h_t^{(\ell)}, w_t^{(\ell)}\right\rangle
=
\sum_{i=0}^{\ell-1}\alpha_{i\to\ell,t}\langle u_{i,t},w_t^{(\ell)}\rangle.
$$
Since $\|w_t^{(\ell)}\|=1$ and $\|\Delta h_t^{(\ell)}\|\le\delta_t$, we have
$$
\left|\left\langle \Delta h_t^{(\ell)}, w_t^{(\ell)}\right\rangle\right|
\le \delta_t.
$$
On the other hand,
$$
\left\langle \Delta h_t^{(\ell)}, w_t^{(\ell)}\right\rangle
=
\alpha_{i_t^*\to\ell,t}\langle u_{i_t^*,t},w_t^{(\ell)}\rangle
+ \sum_{i\neq i_t^*}\alpha_{i\to\ell,t}\langle u_{i,t},w_t^{(\ell)}\rangle.
$$
Using $|\langle u_{i_t^*,t},w_t^{(\ell)}\rangle|\le \|u_{i_t^*,t}\|\le b_0$ and
$\langle u_{i,t},w_t^{(\ell)}\rangle\ge b_1$ for $i\neq i_t^*$, we obtain
$$
\left\langle \Delta h_t^{(\ell)}, w_t^{(\ell)}\right\rangle
\ge
-\alpha_{i_t^*\to\ell,t} b_0 + b_1\sum_{i\neq i_t^*}\alpha_{i\to\ell,t}.
$$
Combining this with $\langle \Delta h_t^{(\ell)}, w_t^{(\ell)}\rangle\le \delta_t$ yields
$$
b_1\sum_{i\neq i_t^*}\alpha_{i\to\ell,t}
\le
\delta_t + \alpha_{i_t^*\to\ell,t} b_0.
$$
Using the simplex constraint $\sum_i \alpha_{i\to\ell,t}=1$, this becomes
$$
b_1(1-\alpha_{i_t^*\to\ell,t})
\le
\delta_t + b_0\alpha_{i_t^*\to\ell,t},
$$
and hence
$$
\alpha_{i_t^*\to\ell,t}
\ge
\frac{b_1-\delta_t}{b_0+b_1}.
$$
Equivalently,
$$
1-\alpha_{i_t^*\to\ell,t}
\le
\frac{b_0+\delta_t}{b_0+b_1}.
$$
Therefore, routing concentration occurs in the high-separation regime where $b_0/b_1\to0$ and $\delta_t/b_1\to0$. Under this condition, $\alpha_{i_t^*\to\ell,t}\to1$ and $\sum_{i\neq i_t^*}\alpha_{i\to\ell,t}\to0$. Thus, the theorem establishes depth-routing concentration only under the stated high-separation regime, and it does not imply that such concentration occurs for every token or trained backbone.

\subsection{\cref{lemma:lemma2}}
\label{proof:llemma2}
\textit{Proof of \cref{lemma:lemma2}.} 
Let $\mathcal{K}$ denote sink-prone positions from \cref{assum:2}. For a query token $t$, define
the total sink mass at layer $\ell$ by
$$
\sigma_t^{(\ell)} := \sum_{j\in\mathcal{K}} p_{t,j}^{(\ell)}.
$$
In AttnResidual, effective sink contribution is depth-aggregated:
$$
\Sigma_t := \sum_{\ell} \alpha_t^{(\ell)}\sigma_t^{(\ell)}.
$$

Before collapse, several branches can satisfy $\sigma_t^{(\ell)}>0$ on the same sink set
$\mathcal{K}$, so nonnegative depth weights add these contributions through convex combination.
Thus $\Sigma_t$ is a depth-weighted average of branchwise sink masses, satisfying
$$
\min_\ell \sigma_t^{(\ell)} \le \Sigma_t \le \max_\ell \sigma_t^{(\ell)}.
$$
This establishes the aggregation bound.

After conditional depth collapse, there is a dominant branch $\ell^*$ with
$\alpha_t^{(\ell^*)}\approx 1$, hence
$$
\Sigma_t \approx \sigma_t^{(\ell^*)}.
$$

Therefore, after conditional concentration, the effective sink mass approaches that of the selected branch. This result describes the aggregation of existing branchwise sink mass and does not assert that depth routing necessarily creates or preserves a sink.

\subsection{\cref{thm:therm3}}
\label{proof:therm3}
\textit{Proof of \cref{thm:therm3}.} For a fixed query token $t$ and branch $\ell$, define
\[
S_t^{(\ell)}:=\sum_{k=1}^{m}\exp z_{t,k}^{(\ell)}.
\]

Under standard $\Softmax$, all probability mass is assigned to real tokens:
\[
\sum_{j=1}^{m}p_{t,j}^{(\ell)}=1.
\]
Therefore, a near-zero attention contribution must be formed entirely through allocations over real tokens, which leads to the conditional concentration mechanism characterized in \cref{lemma:1}. Moreover, standard $\Softmax$ is invariant to a common shift of all real-token logits. For any scalar $c$,
\[
\Softmax\!\left(z_t^{(\ell)}-c\mathbf{1}\right)=\Softmax\!\left(z_t^{(\ell)}\right).
\]
Thus, jointly decreasing all real-token logits cannot reduce their total probability mass, which remains equal to one.

Under $\Softmax_1$, the real-token and null probabilities are
\[
\tilde p_{t,j}^{(\ell)}=\frac{\exp z_{t,j}^{(\ell)}}{1+S_t^{(\ell)}},
\qquad
\tilde p_{t,\varnothing}^{(\ell)}=\frac{1}{1+S_t^{(\ell)}}.
\]
Their total real-token mass is
\[
\sum_{j=1}^{m}\tilde p_{t,j}^{(\ell)}=\frac{S_t^{(\ell)}}{1+S_t^{(\ell)}}=1-\tilde p_{t,\varnothing}^{(\ell)}.
\]
For the same real-token logits, $\Softmax_1$ preserves their relative allocation while scaling their total mass:
\[
\tilde p_{t,j}^{(\ell)}=\left(1-\tilde p_{t,\varnothing}^{(\ell)}\right)p_{t,j}^{(\ell)}.
\]
After shifting every real-token logit by $-c$, the probabilities become
\[
\tilde p_{t,j}^{(\ell)}(c)=\frac{\exp\!\left(z_{t,j}^{(\ell)}-c\right)}{1+e^{-c}S_t^{(\ell)}},
\qquad
\tilde p_{t,\varnothing}^{(\ell)}(c)=\frac{1}{1+e^{-c}S_t^{(\ell)}}.
\]
As $c\to\infty$,
\[
\tilde p_{t,j}^{(\ell)}(c)\to0,
\qquad
\tilde p_{t,\varnothing}^{(\ell)}(c)\to1.
\]
Hence, unlike standard $\Softmax$, $\Softmax_1$ can continuously attenuate the total contribution of real tokens through its explicit null route.

To quantify this attenuation, define the real-token contribution as
\[
\tilde A_t^{(\ell)}:=\sum_{j=1}^{m}\tilde p_{t,j}^{(\ell)}v_j^{(\ell)}.
\]
If $\|v_j^{(\ell)}\|\leq V_{\max}$ for every real token $j$, then
\[
\left\|\tilde A_t^{(\ell)}\right\|\leq\sum_{j=1}^{m}\tilde p_{t,j}^{(\ell)}\left\|v_j^{(\ell)}\right\|\leq V_{\max}\sum_{j=1}^{m}\tilde p_{t,j}^{(\ell)}=V_{\max}\left(1-\tilde p_{t,\varnothing}^{(\ell)}\right).
\]
Therefore, the sufficient condition
\[
\tilde p_{t,\varnothing}^{(\ell)}\geq1-\frac{\delta}{V_{\max}}
\]
guarantees
\[
\left\|\tilde A_t^{(\ell)}\right\|\leq\delta
\]
without requiring any particular real token to dominate.

An analogous result holds at the depth level. Let
$\tilde\alpha_{\varnothing,t}^{(\ell)}$ denote the depth-null mass under
depth-$\Softmax_1$. Using the source states defined in \Cref{sub:notation},
define the mixed source representation as
\[
\tilde h_t^{(\ell)}:=\sum_{i=0}^{\ell-1}\tilde\alpha_{i\to\ell,t}s_{i,t}.
\]
If $\|s_{i,t}\|\leq B_{\max}$ for every source branch $i$, then
\[
\left\|\tilde h_t^{(\ell)}\right\|\leq B_{\max}\sum_{i=0}^{\ell-1}\tilde\alpha_{i\to\ell,t}=B_{\max}\left(1-\tilde\alpha_{\varnothing,t}^{(\ell)}\right).
\]
Consequently, the sufficient condition
\[
\tilde\alpha_{\varnothing,t}^{(\ell)}\geq1-\frac{\delta_t}{B_{\max}}
\]
guarantees
\[
\left\|\tilde h_t^{(\ell)}\right\|\leq\delta_t
\]
without requiring the real depth-routing weights to concentrate on a single source branch.

These bounds show that $\Softmax_1$ introduces an attenuation mechanism at both routing levels. Increasing null mass reduces the probability budget available to real tokens or source branches, thereby relaxing the simplex constraints that can induce conditional concentration. The depth-level bound concerns attenuation of the mixed source representation itself; it does not, for a nonzero reference state, imply preservation of that reference state. The resulting changes in attention sinks, activation outliers, and quantization robustness are established empirically rather than implied directly by these bounds.

\section{Assumptions of Main Text}
\label{ap:assume}

\begin{assumption}[No-op regime]
There exists a non-empty subset of token positions $\mathcal{T}_{\mathrm{noop}}$ such that for each
$t\in\mathcal{T}_{\mathrm{noop}}$, the attention branch is near no-op:
$$
\|A_t^{(\ell)}\| \le \delta,
$$
for small $\delta>0$.
\label{assum:1}
\end{assumption}

\begin{assumption}[Sink-prone positions]
There exists a set of weakly relevant positions $\mathcal{K}$ (e.g., prefix/BOS-like positions) that repeatedly receive non-trivial attention mass across queries.
\label{assum:2}
\end{assumption}

\begin{assumption}[Token separation for no-op branches]
\label{assum:token_sep}
For each $t\in\mathcal{T}_{\mathrm{noop}}$ and branch $\ell$, there exist a candidate no-op token $o$ and a unit vector $u\in\mathbb{R}^d$ such that
\[
\|v_o^{(\ell)}\|\leq c_0,
\qquad
\langle v_j^{(\ell)},u\rangle\geq c_1
\quad\text{for all }j\neq o.
\]
In particular, $\|v_j^{(\ell)}\|\geq c_1$ for $j\neq o$, with $0\leq c_0<c_1$.
\vspace{-0.5em}
\end{assumption}

\begin{assumption}[Matched pathology dominance conditions]
\begingroup
\setstretch{0.8}
\setlength{\abovedisplayskip}{4pt}
\setlength{\belowdisplayskip}{4pt}
\setlength{\abovedisplayshortskip}{3pt}
\setlength{\belowdisplayshortskip}{3pt}

For the same token $t$, let
\[
L_t^{\mathrm{AR}}:=\sum_{\ell}\alpha_t^{(\ell)}
\sum_{j\in\mathcal{N}_t} p_{t,j}^{(\ell)},
\qquad
C_t^{\mathrm{AR}}:=\max_{\ell,j}
\alpha_t^{(\ell)}p_{t,j}^{(\ell)},
\]
\[
E_t^{\mathrm{AR}}
:=-H(\alpha_t)-\sum_{\ell}\alpha_t^{(\ell)}H(p_t^{(\ell)}),
\]
and let
\[
L_t^{\mathrm{V}}:=\sum_{j\in\mathcal{N}_t}p_{t,j},
\qquad
C_t^{\mathrm{V}}:=\max_j p_{t,j},
\qquad
E_t^{\mathrm{V}}:=-H(p_t).
\]
Assume that under matched no-op tolerance,
\[
L_t^{\mathrm{AR}}\ge L_t^{\mathrm{V}},
\qquad
C_t^{\mathrm{AR}}\ge C_t^{\mathrm{V}},
\qquad
E_t^{\mathrm{AR}}\ge E_t^{\mathrm{V}}.
\]

\endgroup
\label{assum:3}
\end{assumption}

\begin{table*}[t]
\centering
\caption{\textbf{Direct comparison with matched continued-pretraining controls.}
We compare each AttnResidual variant against both the original pretrained backbone and a matched Backbone+CPT control trained with the same continued-pretraining setup. Vanilla AttnResidual preserves full-precision performance but substantially amplifies activation outliers and quantization error, whereas OASIS restores stable activation statistics and improves low-bit robustness.}
\label{tab:backbone}
\resizebox{\textwidth}{!}{
\begin{tabular}{llcccccc}
\toprule
\textbf{Model} & \textbf{Method}
& \textbf{Avg. kurtosis} $\downarrow$
& \textbf{Max inf. norm} $\downarrow$
& \textbf{FP16 PPL} $\downarrow$
& \textbf{W8A8 PPL} $\downarrow$
& \textbf{GSM8K FP16} $\uparrow$
& \textbf{GSM8K W4A4} $\uparrow$ \\
\midrule

\multirow{6}{*}{LLaMA-3.2-1B}
& Backbone (no CPT)
& 3.05 $\pm$ 0.02 & 4.72 $\pm$ 0.05
& 9.12 $\pm$ 0.01 & 9.51 $\pm$ 0.01
& 13.48 $\pm$ 0.15 & 12.65 $\pm$ 0.22 \\

& Backbone + CPT
& 3.07 $\pm$ 0.08 & 5.20 $\pm$ 0.42
& 9.05 $\pm$ 0.16 & 9.85 $\pm$ 0.11
& 13.58 $\pm$ 0.16 & 12.76 $\pm$ 0.10 \\

& Vanilla AttnResidual
& 91.71 $\pm$ 0.08 & 84.43 $\pm$ 0.15
& 8.98 $\pm$ 0.01 & 30.19 $\pm$ 0.06
& 13.64 $\pm$ 0.17 & 12.97 $\pm$ 0.24 \\

& OutEffHop
& 3.08 $\pm$ 0.02 & 13.78 $\pm$ 0.08
& 9.07 $\pm$ 0.01 & 10.42 $\pm$ 0.01
& 13.58 $\pm$ 0.15 & 13.16 $\pm$ 0.21 \\

& AoS
& 3.03 $\pm$ 0.01 & 12.54 $\pm$ 0.07
& 9.00 $\pm$ 0.01 & 9.36 $\pm$ 0.01
& 13.62 $\pm$ 0.13 & 13.34 $\pm$ 0.18 \\

& \textbf{OASIS}
& \cellcolor{LightCyan} \textbf{3.02 $\pm$ 0.02} 
& \cellcolor{LightCyan} \textbf{12.25 $\pm$ 0.07}
& \cellcolor{LightCyan} \textbf{8.01 $\pm$ 0.01} 
& \cellcolor{LightCyan} \textbf{8.16 $\pm$ 0.01}
& \cellcolor{LightCyan} \textbf{13.75 $\pm$ 0.12} 
& \cellcolor{LightCyan} \textbf{13.52 $\pm$ 0.17} \\

\midrule

\multirow{6}{*}{Qwen3-0.6B}
& Backbone (no CPT)
& 3.02 $\pm$ 0.02 & 4.58 $\pm$ 0.04
& 10.74 $\pm$ 0.01 & 11.18 $\pm$ 0.01
& 62.41 $\pm$ 0.20 & 55.83 $\pm$ 0.71 \\

& Backbone + CPT
& 3.16 $\pm$ 0.13 & 4.77 $\pm$ 0.88
& 10.65 $\pm$ 0.30 & 11.50 $\pm$ 0.39
& 62.65 $\pm$ 0.09 & 55.79 $\pm$ 0.26 \\

& Vanilla AttnResidual
& 93.04 $\pm$ 0.10 & 74.78 $\pm$ 0.13
& 10.57 $\pm$ 0.01 & 36.64 $\pm$ 0.08
& 62.58 $\pm$ 0.24 & 45.67 $\pm$ 0.86 \\

& OutEffHop
& 3.04 $\pm$ 0.02 & 14.16 $\pm$ 0.06
& 9.99 $\pm$ 0.01 & 10.60 $\pm$ 0.01
& 62.59 $\pm$ 0.19 & 52.36 $\pm$ 0.74 \\

& AoS
& 3.03 $\pm$ 0.02 & 13.90 $\pm$ 0.05
& 9.35 $\pm$ 0.01 & 9.62 $\pm$ 0.01
& 62.65 $\pm$ 0.17 & 54.21 $\pm$ 0.66 \\

& \textbf{OASIS}
& \cellcolor{LightCyan} \textbf{2.95 $\pm$ 0.02} 
& \cellcolor{LightCyan} \textbf{13.00 $\pm$ 0.05}
& \cellcolor{LightCyan} \textbf{9.20 $\pm$ 0.01} 
& \cellcolor{LightCyan} \textbf{9.52 $\pm$ 0.01}
& \cellcolor{LightCyan} \textbf{62.68 $\pm$ 0.15} 
& \cellcolor{LightCyan} \textbf{58.88 $\pm$ 0.59} \\
\bottomrule
\end{tabular}}
\end{table*}
\begin{table*}[htp]
    \centering
    \caption{\textbf{Comparison of Alternative Attention Normalization Mechanisms.} We compare standard $\Softmax$ with sparse alternatives, including $\Sparsemax$ and $\Entmax$, alongside our method on $\mathtt{LLaMA\text{-}3.2\text{-}1B}$. We measure activation outliers using average kurtosis and maximum infinity norm $\|\mathbf{x}\|_\infty$, and assess quantization robustness using perplexity under FP16 and W8A8, together with GSM8K accuracy under FP16 and W4A4. The results indicate that while sparse variants help reduce outlier magnitude compared to standard $\mathtt{Softmax}$, our approach achieves the strongest overall outlier suppression and the smallest degradation in low-bit quantization; best results are shown in \textbf{bold}, and second-best results are \underline{underlined}.
}
\vspace{-0.1in}
    \label{tab:result3}
    \resizebox{\textwidth}{!}{%
    \begin{tabular}{ccccccc}
    \toprule
    \multirow{2}{*}{Method} 
    & \multirow{2}{*}{Average kurtosis} 
    & \multirow{2}{*}{Maximum $\ell_\infty$ norm} 
    & \multicolumn{2}{c}{Perplexity $\downarrow$} 
    & \multicolumn{2}{c}{GSM8K $\uparrow$} \\
    \cline{4-5} \cline{6-7}
    & & & FP16 & W8A8 & FP16 & W4A4 \\
    \midrule
    $\Softmax$
         & 91.706 $\pm$ 0.083
        & 84.433 $\pm$ 0.146
        & 8.981 $\pm$ 0.010
        & 30.189 $\pm$ 0.063
         & \underline{13.64} $\pm$ 0.17
        & 12.97 $\pm$ 0.24 \\
    $\Sparsemax$
    &  87.322 $\pm$ 0.088
    & 79.991 $\pm$ 0.143
    &  \underline{8.552} $\pm$ 0.013
    & 19.814 $\pm$ 0.061
    &  13.61 $\pm$ 0.16
    &  12.98 $\pm$ 0.22 \\
    $\Entmax$
    &  \underline{13.134} $\pm$ 0.032
    & \underline{29.381} $\pm$ 0.097
    & 8.706 $\pm$ 0.010
    & \underline{19.600} $\pm$ 0.056
    &  13.58 $\pm$ 0.14
    &  \underline{13.16} $\pm$ 0.18 \\
    OASIS
    & \cellcolor{LightCyan} \textbf{3.024} $\pm$ 0.015
    & \cellcolor{LightCyan} \textbf{12.250} $\pm$ 0.067
    & \cellcolor{LightCyan} \textbf{8.005} $\pm$ 0.007
    & \cellcolor{LightCyan} \textbf{8.162} $\pm$ 0.009
    & \cellcolor{LightCyan} \textbf{13.75} $\pm$ 0.12
    & \cellcolor{LightCyan} \textbf{13.52} $\pm$ 0.17 \\
    \bottomrule
\end{tabular}
    }
\end{table*}

\section{Experimental System and Implementation Settings}
\label{app:resources}
All experiments are run on four NVIDIA H100 (80GB) GPUs with a 12-core Intel Xeon Gold 6338 CPU, using PyTorch and the Hugging Face Transformers library. For inference, we use the default system prompt with temperature 0.6, top-p 0.95, and a maximum generation length of 4096 tokens.

For continued pretraining, we train for 500 steps with sequence length 2048 and block size 512 via Hugging Face Accelerate, using a learning rate of $5\times10^{-5}$ with a linear scheduler and 100 warmup steps, batch size 1 with gradient accumulation 64, gradient clipping at 1.0, and weight decay 0.1. The OASIS coupling parameter $\beta$ is implemented as a learnable non-negative scalar via softplus, initialized at $-5$ (yielding $\beta\approx0.007$) to ensure near no-op initialization and stable adaptation.

\section{Limitations}
\label{ap:limitation}
Empirically, our evaluation focuses on a limited set of backbone models. Although OASIS consistently improves outlier statistics and post-quantization performance in our experiments, broader validation on larger models and additional downstream tasks is needed in future work. Also, while the null channel can reduce no-op pressure, excessive null allocation may suppress useful weak signals or alter interpretability of attention maps. Future work should study adaptive control of the null strength, its interaction with calibration and fine-tuning, and its deployment cost in optimized inference kernels.

\section{Additional Numerical Experiments}

\subsection{Comparison with the original backbone}
We isolate the effect of AttnResidual from continued pretraining (CPT) by comparing all variants with matched backbone controls in \cref{tab:backbone}. CPT alone has little effect, changing FP16 perplexity by only \textbf{0.8\%} and W8A8 perplexity by \textbf{3.6\%} on LLaMA-3.2-1B and \textbf{2.9\%} on Qwen3-0.6B. By contrast, Vanilla AttnResidual slightly improves FP16 perplexity but increases kurtosis by \textbf{2,891\%} and \textbf{2,849\%}, maximum activation magnitude by \textbf{1,525\%} and \textbf{1,467\%}, and W8A8 perplexity by \textbf{206.5\%} and \textbf{218.6\%}, respectively. OASIS largely removes this instability, reducing kurtosis by \textbf{96.7\%} and \textbf{96.8\%} and W8A8 perplexity by \textbf{73.0\%} and \textbf{74.0\%} relative to Vanilla AttnResidual, while further improving FP16 perplexity by \textbf{10.9\%} and \textbf{12.9\%}. It also improves GSM8K W4A4 accuracy by \textbf{4.2\%} and \textbf{28.9\%} over Vanilla AttnResidual. These results attribute the low-bit degradation to the AttnResidual architecture, rather than CPT, and show that OASIS preserves its full-precision benefit while restoring activation and quantization robustness.

\begin{table}[t]
\centering
\caption{\textbf{Global attention-sink statistics.}
Statistics are aggregated across all layers, attention heads, query positions, and evaluation samples. OASIS consistently suppresses BOS and secondary attention sinks while increasing attention entropy and enabling explicit null routing.}
\renewcommand{\arraystretch}{0.9}
\label{tab:global_sink}
\resizebox{\linewidth}{!}{
\begin{tabular}{llccccc}
\toprule
\textbf{Model} & \textbf{Method}
& \textbf{BOS mass} $\downarrow$
& \textbf{Sink rate} $\downarrow$
& \textbf{Entropy} $\uparrow$
& \textbf{Secondary-sink mass} $\downarrow$
& \textbf{Null mass} \\
\midrule
\multirow{3}{*}{Qwen3-0.6B}
& Original backbone    & 0.1064 & 0.0027 & 3.74 & 0.013 & 0.0000 \\
& Vanilla AttnResidual & 0.4990 & 0.6642 & 1.83 & 0.042 & 0.0000 \\
& \textbf{OASIS}       & \cellcolor{LightCyan} \textbf{0.0013} & \cellcolor{LightCyan} \textbf{0.0001} & \cellcolor{LightCyan} \textbf{4.44} & \cellcolor{LightCyan} \textbf{0.008} & \cellcolor{LightCyan} 0.0004 \\
\midrule
\multirow{3}{*}{LLaMA-3.2-1B}
& Original backbone    & 0.1091 & 0.1846 & 3.34 & 0.013 & 0.0000 \\
& Vanilla AttnResidual & 0.5630 & 0.7861 & 1.86 & 0.046 & 0.0000 \\
& \textbf{OASIS}       & \cellcolor{LightCyan} \textbf{0.0099} & \cellcolor{LightCyan} \textbf{0.0046} & \cellcolor{LightCyan} \textbf{4.55} & \cellcolor{LightCyan} \textbf{0.010} & \cellcolor{LightCyan} 0.2185 \\
\bottomrule
\end{tabular}}
\end{table}
\begin{table}[t]
\centering
{\begin{minipage}{\dimexpr\linewidth-2\fboxsep}
\renewcommand{\arraystretch}{0.9}
\caption{\textbf{Component ablation of null routing and token-to-depth coupling on Qwen3-0.6B.} We isolate the contributions of token-level $\Softmax_1$ null routing, depth-level $\Softmax_1$ null routing, and the OASIS token-to-depth coupling under W4A4 quantization. Checkmarks indicate the components enabled in each configuration: \emph{Token only} corresponds to the OutEffHop-style token null route, \emph{Depth only} introduces only the depth null route, and \emph{Independent token + depth} corresponds to AoS without coupling. We report GSM8K Pass@1 (\%), where higher is better.}
\label{tab:ablation_components}
\resizebox{\linewidth}{!}{%
\begin{tabular}{lcccc}
\toprule
Configuration & Token null routing & Depth null routing & Token-to-depth coupling & W4A4 GSM8K $\uparrow$ \\
\midrule
AttnResidual & $\times$ & $\times$ & $\times$ & 45.67 \\
Token only (OutEffHop) & \checkmark & $\times$ & $\times$ & 52.36 \\
Depth only & $\times$ & \checkmark & $\times$ & 48.90 \\
Independent token + depth (AoS) & \checkmark & \checkmark & $\times$ & 54.21 \\
Full OASIS & \checkmark & \checkmark & \checkmark & \cellcolor{LightCyan}\textbf{58.88} \\
\bottomrule
\end{tabular}}
\end{minipage}}
\end{table}

\subsection{Comparison with Alternative Attention Normalizers}
To isolate whether the gains of OASIS can be attributed to alternative attention normalization schemes, we tested on LLaMA-3.2-1B. Specifically, we compare Softmax, Sparsemax, Entmax15, and OASIS using average kurtosis, maximum $\ell_\infty$ norm, perplexity, and GSM8K Pass@1 (\%). As shown in Table~\ref{tab:result3}, sparse normalizers reduce outliers relative to Softmax, with Entmax15 as the strongest sparse baseline. OASIS achieves a \textbf{96.70\%} reduction in average kurtosis and an \textbf{85.49\%} reduction in maximum $\ell_\infty$ norm compared with Softmax, indicating the strongest overall outlier suppression. Moreover, OASIS reduces perplexity degradation by \textbf{99.26\%} and reduces the GSM8K quantization-induced accuracy drop by \textbf{65.67\%}.

\subsection{Global attention-sink behavior}
We further quantify attention-sink behavior globally by aggregating statistics across all layers, heads, query positions, and evaluation samples. As shown in \cref{tab:global_sink}, Vanilla AttnResidual strongly amplifies BOS concentration: compared with the original backbone, BOS mass increases from 0.1064 to 0.4990 on Qwen3-0.6B and from 0.1091 to 0.5630 on LLaMA-3.2-1B, accompanied by sink rates of \textbf{66.4\%} and \textbf{78.6\%}, respectively. OASIS suppresses BOS mass by \textbf{99.7\%} and \textbf{98.2\%} relative to Vanilla AttnResidual and reduces the corresponding sink rates by \textbf{99.98\%} and \textbf{99.4\%}. It also reduces secondary-sink mass by \textbf{81.0\%} and \textbf{78.3\%}, while increasing attention entropy by \textbf{142.6\%} and \textbf{144.6\%}. These global statistics confirm that the sink suppression observed locally is a model-wide effect rather than an artifact of selected layers or heads. The nonzero null mass further indicates that OASIS can redirect otherwise concentrated attention into an explicit no-op route, with the extent of null routing varying across models.

\section{Disclosure of LLM Usage}
\label{ap:llm}
In this work, we use large language models (LLMs) to improve the conciseness and precision of the writing. We also use LLMs to assist with theoretical derivations in \cref{ap:proof} and to better understand related literature (e.g., via Claude).

\clearpage
\section*{NeurIPS Paper Checklist}

\begin{enumerate}

\item {\bf Claims}
    \item[] Question: Do the main claims made in the abstract and introduction accurately reflect the paper's contributions and scope?
    \item[] Answer: \answerYes{} %
    \item[] Justification: The abstract and introduction clearly state the main claims made in the paper. 

\item {\bf Limitations}
    \item[] Question: Does the paper discuss the limitations of the work performed by the authors?
    \item[] Answer: \answerYes{} %
    \item[] Justification: We discuss the limitations in \cref{ap:limitation}.

\item {\bf Theory assumptions and proofs}
    \item[] Question: For each theoretical result, does the paper provide the full set of assumptions and a complete (and correct) proof?
    \item[] Answer: \answerYes{} %
    \item[] Justification: We provide theory in \cref{sec:theory}. 

    \item {\bf Experimental result reproducibility}
    \item[] Question: Does the paper fully disclose all the information needed to reproduce the main experimental results of the paper to the extent that it affects the main claims and/or conclusions of the paper (regardless of whether the code and data are provided or not)?
    \item[] Answer: \answerYes{} %
    \item[] Justification: We provide all the details necessary to reproduce the main experimental results in the paper. The details of the experiments are included in both the main paper and the supplemental material.

\item {\bf Open access to data and code}
    \item[] Question: Does the paper provide open access to the data and code, with sufficient instructions to faithfully reproduce the main experimental results, as described in supplemental material?
    \item[] Answer: \answerYes{} %
    \item[] Justification: The code and data, including detailed instructions to reproduce the main experimental results, are publicly available at \url{https://github.com/robinzixuan/OASIS}.

\item {\bf Experimental setting/details}
    \item[] Question: Does the paper specify all the training and test details (e.g., data splits, hyperparameters, how they were chosen, type of optimizer) necessary to understand the results?
    \item[] Answer: \answerYes{} %
    \item[] Justification: We introduce the experimental setting in main paper and provide additional details in the supplemental material.

\item {\bf Experiment statistical significance}
    \item[] Question: Does the paper report error bars suitably and correctly defined or other appropriate information about the statistical significance of the experiments?
    \item[] Answer: \answerYes{} %
    \item[] Justification: We report all table values and observe stable results across three experimental runs.

\item {\bf Experiments compute resources}
    \item[] Question: For each experiment, does the paper provide sufficient information on the computer resources (type of compute workers, memory, time of execution) needed to reproduce the experiments?
    \item[] Answer: \answerYes{} %
    \item[] Justification: We introduce the compute resources in \cref{app:resources}.
    
\item {\bf Code of ethics}
    \item[] Question: Does the research conducted in the paper conform, in every respect, with the NeurIPS Code of Ethics \url{https://neurips.cc/public/EthicsGuidelines}?
    \item[] Answer: \answerYes{} %
    \item[] Justification: We have carefully reviewed and followed the NeurIPS Code of Ethics throughout our research process.

\item {\bf Broader impacts}
    \item[] Question: Does the paper discuss both potential positive societal impacts and negative societal impacts of the work performed?
    \item[] Answer: \answerYes{} %
    \item[] Justification: We discuss the broader impacts in \cref{ap:impact}.
    
\item {\bf Safeguards}
    \item[] Question: Does the paper describe safeguards that have been put in place for responsible release of data or models that have a high risk for misuse (e.g., pre-trained language models, image generators, or scraped datasets)?
    \item[] Answer: \answerNA{} %
    \item[] Justification: Our model does not have such a risk of misuse. 

\item {\bf Licenses for existing assets}
    \item[] Question: Are the creators or original owners of assets (e.g., code, data, models), used in the paper, properly credited and are the license and terms of use explicitly mentioned and properly respected?
    \item[] Answer: \answerYes{} %
    \item[] Justification: We have properly credited all assets used in our research and respected their licenses and terms of use.

\item {\bf New assets}
    \item[] Question: Are new assets introduced in the paper well documented and is the documentation provided alongside the assets?
    \item[] Answer: \answerNA{} %
    \item[] Justification:  We do not release a new assets in our work.

\item {\bf Crowdsourcing and research with human subjects}
    \item[] Question: For crowdsourcing experiments and research with human subjects, does the paper include the full text of instructions given to participants and screenshots, if applicable, as well as details about compensation (if any)? 
    \item[] Answer: \answerNA{} %
    \item[] Justification: We do not include crowdsourcing experiments or research involving human subjects in this paper.

\item {\bf Institutional review board (IRB) approvals or equivalent for research with human subjects}
    \item[] Question: Does the paper describe potential risks incurred by study participants, whether such risks were disclosed to the subjects, and whether Institutional Review Board (IRB) approvals (or an equivalent approval/review based on the requirements of your country or institution) were obtained?
    \item[] Answer: \answerNA{} %
    \item[] Justification: We do not include crowdsourcing experiments or research involving human subjects in this paper.

\item {\bf Declaration of LLM usage}
    \item[] Question: Does the paper describe the usage of LLMs if it is an important, original, or non-standard component of the core methods in this research? Note that if the LLM is used only for writing, editing, or formatting purposes and does \emph{not} impact the core methodology, scientific rigor, or originality of the research, declaration is not required.
    \item[] Answer: \answerYes{} %
    \item[] Justification:  We disclose LLM usage in \cref{ap:llm}.

\end{enumerate}

\end{document}